\documentclass{article}

\PassOptionsToPackage{table}{xcolor}
\usepackage{etoolbox}       %

\usepackage{graphicx}  %
\usepackage{anyfontsize}

\usepackage{times}
\usepackage[parfill]{parskip}
\usepackage{amsmath,amsthm}
\usepackage{mathtools}  %
\usepackage{adjustbox}
\usepackage{threeparttable}
\usepackage{tablefootnote}
\usepackage{afterpage}

\newlength{\defbaselineskip}
\RequirePackage[T1]{fontenc}
\RequirePackage[tt=false, type1=true]{libertine}
\RequirePackage[varqu]{zi4}
\RequirePackage[libertine]{newtxmath}

\usepackage{bm}

\usepackage[inline]{enumitem}
\usepackage{booktabs}
\usepackage{multirow}
\usepackage{wrapfig}
\usepackage{algorithm,algorithmicx,algpseudocode}
\usepackage{nicematrix}

\usepackage{amsmath,amsfonts,bm}

\newcommand{\figleft}{{\em (Left)}}

\newcommand{\figright}{{\em (Right)}}

\newcommand{\captiona}{{\em (a)}}
\newcommand{\captionb}{{\em (b)}}
\newcommand{\captionc}{{\em (c)}}

\def\Figref#1{Figure~\ref{#1}}

\def\eqref#1{equation~\ref{#1}}
\def\Eqref#1{Equation~\ref{#1}}

\def\1{\bm{1}}

\def\rmI{{\mathbf{I}}}

\def\rmS{{\mathbf{S}}}

\def\vone{{\boldsymbol{1}}}
\def\vmu{{\boldsymbol{\mu}}}

\def\vlambda{{\boldsymbol{\lambda}}}
\def\vbeta{{\boldsymbol{\beta}}}

\def\vphi{{\boldsymbol{\phi}}}

\def\MOmega{{\boldsymbol{\Omega}}}

\def\va{{\boldsymbol{a}}}

\def\vc{{\boldsymbol{c}}}
\def\vd{{\boldsymbol{d}}}
\def\ve{{\boldsymbol{e}}}
\def\vf{{\boldsymbol{f}}}

\def\vh{{\boldsymbol{h}}}

\def\vk{{\boldsymbol{k}}}

\def\vq{{\boldsymbol{q}}}
\def\vr{{\boldsymbol{r}}}
\def\vs{{\boldsymbol{s}}}

\def\vv{{\boldsymbol{v}}}

\def\vx{{\boldsymbol{x}}}
\def\vy{{\boldsymbol{y}}}
\def\vz{{\boldsymbol{z}}}

\def\vrep{{\boldsymbol{k}}}
\def\rep{{k}}

\def\mD{{\bm{D}}}

\def\mF{{\bm{F}}}

\def\mH{{\bm{H}}}

\def\mK{{\bm{K}}}

\def\mP{{\bm{P}}}
\def\mQ{{\bm{Q}}}

\def\mS{{\bm{S}}}

\def\mY{{\bm{Y}}}

\def\mLambda{{\bm{\Lambda}}}

\DeclareMathAlphabet{\mathsfit}{\encodingdefault}{\sfdefault}{m}{sl}
\SetMathAlphabet{\mathsfit}{bold}{\encodingdefault}{\sfdefault}{bx}{n}

\newcommand{\R}{\mathbb{R}}

\usepackage{hyperref}
\usepackage{url}

\usepackage{natbib}
\PassOptionsToPackage{numbers, compress}{natbib}

\usepackage[utf8]{inputenc} %
\usepackage[T1]{fontenc}    %
\usepackage{booktabs}       %
\usepackage{amsfonts}       %
\usepackage{nicefrac}       %
\usepackage{microtype}      %
\usepackage[table]{xcolor}

\usepackage{microtype}
\usepackage{subfigure}

\usepackage{wrapfig}

\usepackage{mathrsfs}
\usepackage{amsmath}
\usepackage{caption}
\usepackage{capt-of}
\usepackage{lipsum}
\usepackage{multibib}
\usepackage{adjustbox}
\usepackage{nicefrac}

\usepackage{tikz}
\usepackage{tikz-3dplot}
\usetikzlibrary{decorations.markings}

\usepackage{tikz}
\usepackage{tikz-3dplot}
\usepackage{ifthen}

\usepackage{pgfplots}
\usepackage{pgfplotstable}
\usepgfplotslibrary{groupplots}
\pgfplotsset{compat=1.16} %
\usetikzlibrary{calc}

\usepackage{tcolorbox}
\newtcolorbox{highlightbox}{
  colback=green!12,    %
  colframe=green!120!black, %
  boxrule=0.5pt,        %
  arc=6pt,              %
  left=6pt,             %
  right=6pt,            %
  top=6pt,              %
  bottom=6pt,           %
  boxsep=0pt,           %
}

\usepackage{booktabs}
\definecolor{c1}{HTML}{765D97} 
\definecolor{c2}{HTML}{900C3F}
\definecolor{c3}{HTML}{fc6160}
\definecolor{myblue}{HTML}{E6F3FC} 
\definecolor{mygray}{HTML}{DBE2E9} 
\definecolor{mygreen}{HTML}{006400}

\hypersetup{
    colorlinks=true,
    linkcolor=c1,
    urlcolor=c1,
    citecolor=c2,
}

\newtheorem{theorem}{Theorem}[section]
\newtheorem{proposition}[theorem]{Proposition}
\newtheorem{remark2}[theorem]{Remark}

\usepackage{multirow}

\newcommand{\op}[1]{\operatorname{#1}}
\newcommand{\method}[1]{$\operatorname{#1}$}

\usepackage{tikz}

\definecolor{dark2orange}{rgb}{0.9, 0.4, 0.}
\definecolor{dark2purple}{rgb}{0.4, 0.4, 0.8}

\usepackage{mathtools}

\newcommand{\eq}[1]{(\ref{#1})}

\usepackage{xspace}
\newcommand{\model}{Lattice\xspace}

\renewcommand{\exp}{\texttt{exp}}
\renewcommand{\log}{\texttt{log}}
\renewcommand{\captionsize}{\footnotesize}

\newcommand{\nxt}{t}
\newcommand{\cur}{t-1}

\usepackage{listings}

\definecolor{codegreen}{rgb}{0,0.5,0}
\definecolor{codegray}{rgb}{0.5,0.5,0.5}
\definecolor{codepurple}{rgb}{0.58,0,0.82}
\definecolor{backcolour}{rgb}{1,1,1}

\lstdefinestyle{mystyle}{
    backgroundcolor=\color{backcolour},   
    keywordstyle=\color{magenta},
    numberstyle=\tiny\color{codegray},
    commentstyle=\color{green!70!blue}, %
    stringstyle=\color{orange}, %
    basicstyle=\ttfamily\scriptsize,
    breakatwhitespace=false,         
    breaklines=false,                 
    captionpos=b,                    
    keepspaces=true,                 
    numbers=left,                    
    numbersep=3pt,                  
    showspaces=false,                
    showstringspaces=false,
    showtabs=false,                  
    tabsize=2,
    escapeinside={(*}{*)}, %
}

\usepackage[textsize=tiny]{todonotes}

\newcommand{\TODO}[1]{}
 \newcommand{\todoM}[2][]{}
\newcommand{\todoV}[2][]{}
\newcommand{\mahdi}[1]{}

\newcommand{\idea}[1]{}
\newcommand{\future}[1]{}    %
\newcommand{\MAYBE}[1]{} %

\title{
\textsc{\model}: Learning to Efficiently Compress the Memory
}

\newcommand\blfootnote[1]{%
  \begingroup
  \renewcommand\thefootnote{}\footnote{#1}%
  \addtocounter{footnote}{-1}%
  \endgroup
}

\author{
Mahdi Karami$^{1 ~*}$ \quad  Razvan Pascanu$^{2}$ \quad  Vahab Mirrokni$^1$ \\
$^1$Google Research \quad $^2$Google DeepMind\\
\protect \blfootnote{\texttt{\{mahdika, razp, and mirrokni\}@google.com}}
}
\date{}

\begin{document}
\maketitle

\begin{abstract}
Attention mechanisms have revolutionized sequence learning but suffer from quadratic computational complexity. This paper introduces \model, a novel recurrent neural network (RNN) mechanism that leverages the inherent low-rank structure of K-V matrices to efficiently compress the cache into a fixed number of memory slots, achieving sub-quadratic complexity. 
We formulate this compression as an online optimization problem and derive a dynamic memory update rule based on a single gradient descent step.  The resulting recurrence features a state- and input-dependent gating mechanism, offering an interpretable memory update process. 
The core innovation is the orthogonal update: each memory slot is updated exclusively with information orthogonal to its current state, hence incorporating only novel, non-redundant data to minimize interference with previously stored information.
We derive an efficient computation for this orthogonal update rule and further approximate it with chunk-wise parallelization to ensure training scalability.
Empirically, \model outperforms strong baselines on language modeling and associative recall tasks across diverse context lengths and model sizes, achieving superior memory efficiency with significantly reduced memory sizes.

\end{abstract}

\section{Introduction}

The attention mechanism~\citep{vaswani2018attention} has become a cornerstone of sequence modeling, offering significant advantages over traditional recurrent and convolutional approaches. By enabling models to dynamically attend to relevant parts of an input sequence while leveraging parallel computation, it effectively captures long-range dependencies and enables in-context learning. These strengths have driven its widespread adoption across various domains, including natural language processing (NLP)~\citep{devlin2018bert,radford2018improving}, computer vision~\citep{dosovitskiy2020image, arnab2021vivit}, 
and graph structure learning and generation~\citep{yun2019graph,dwivedi2020generalization,karami2024higen,behrouz2024BestOfBoth}. 
However, despite its effectiveness, the quadratic time and space complexity of attention limits its scalability in long sequence modeling.
Additionally, its reliance on an unbounded cache leads to inefficient memory management, further limiting its applicability in resource-constrained settings.
These challenges have motivated the development of alternative architectures that aim to retain the expressivity of Transformers while addressing its computational bottlenecks.

Sequence mixing approaches like state space models (SSMs)~\citep{gu2021efficiently, gu2020hippo,gu2022S4D, gu2020hippo, mehta2022long} and linear attention variants~\citep{katharopoulos2020transformers, choromanski2020rethinking} have recently gained renewed interest as promising alternatives to softmax attention. 
While traditional SSMs, with their inherent linear recurrent structure, offer parallelization during training, they often struggle to match the expressivity of standard attention.
Linear attention methods reduce complexity by approximating the attention matrix but can sacrifice accuracy. More recently, input-dependent SSMs~\citep{gu2023mamba, mamba2} and modern gated linear RNNs~\citep{orvieto2023LRU, de2024griffin, beck2024xlstm, peng2025rwkv, yang2024gatedDeltaNet} have demonstrated enhanced expressiveness and improved in-context learning while enabling parallelization through techniques like the associative scan~\citep{blelloch1990prefix, smith2023simplified, de2024griffin}. 
However, a fundamental challenge remains: their ability to efficiently compress and summarize information over very long contexts is often limited by their fixed-size hidden states~\citep{arora2024simple}. Moreover, their linear updates to memory lack efficient mechanisms for selective interaction between stored information and incoming keys, limiting their ability to discard irrelevant or redundant content dynamically.
Non-linear recurrent networks have been revisited in recent works ~\citep{beck2024xlstm,sun2024ttt,titans2024,karami2025trellis}, offering expressive sequence models.
On the other hand, global convolutions~\citep{romero2021ckconv, li2022makes, poli2023hyena} and their input-dependent variants~\citep{karami2019invertible, karami2024orchid} offer another direction for long context modeling by dynamically adapting convolutional filters to the input, but they are not inherently compatible with causal modeling, which is used in autoregressive language generation.

\begin{figure}[t]
\begin{minipage}[t]{0.55\textwidth}
\centering
    \begin{adjustbox}{width=1.\linewidth}
        \begin{tikzpicture}[scale=1.2, every node/.style={font=\small}]

  \draw[->] (-.5,0) -- (3,0) node[right] {$x$};
  \draw[->] (0,-1.) -- (0,3) node[above] {$y$};

  \coordinate (O) at (0,0);
  \coordinate (s) at (1,2);
  \coordinate (v) at (2,1);
  \coordinate (vproj) at (1.2, -0.6);
  
  \coordinate (snext) at (2.2, 1.4);
  
  \coordinate (snext0) at (3., 3.);
  
  \draw[->,  very thick, red] (O) -- (s) node[midway, above left] {$\vs_{\cur}$};
  
  \draw[->, very thick, blue] (O) -- (v) node[below] {$\vh_t$};
  
  \draw[->, very thick, green!70!black] (O) -- (vproj) node[below] {$\vh_t^{\perp {\vs}_{\cur}} = \mathbf{P}({\vs}_{\cur}) \, \vh_t$};
  
  \draw[dashed] (v) -- (vproj);
  
  \draw[dashed, orange, domain=-.5:3] plot (\x, {-0.5*\x}) node[below] {
  \tiny Orth. complement 
  space of $\vs_{\cur}$
  };

\node at (1,-1.8) {\captiona};

\begin{scope}[xshift=4cm]

  \draw[->] (-.5,0) -- (3,0) node[right] {$x$};
  \draw[->] (0,-.5) -- (0,3) node[above] {$y$};

  \coordinate (O) at (0,0);
  \coordinate (s) at (1,2);
  \coordinate (v) at (2,1);
  \coordinate (vproj) at (1.2, -0.6);
  
  \coordinate (snext) at (2.2, 1.4);
  
  \coordinate (snext0) at (3., 3.);
  
  \draw[->,  very thick, red] (O) -- (s) node[midway, above left] {$\vs_{\cur}$};

  \draw[->, very thick, green!70!black] (O) -- (vproj) node[below] {$\vh_t^{\perp {\vs}_{\cur}} = \mathbf{P}({\vs}_{\cur}) \, \vh_t$};
  
  \draw[dashed] (snext) -- (vproj);

  \draw[dashed] (snext) -- (s);

  \draw[->, very thick, purple] (O) -- (snext) node[above] {$\vs_{\nxt}$};
    
  \draw[->, very thick, dashed, violet] (O) -- (snext0) node[above] {$\hat{\vs}_{\nxt}$};
\node at (1.5,-1.8) 
{\captionb};

\end{scope}

\end{tikzpicture}
    \end{adjustbox}
\end{minipage}%
  \quad
\begin{minipage}[t]{0.43\textwidth}
\centering
 \vskip -185pt
    \begin{adjustbox}{width=.9\linewidth}
        \tdplotsetmaincoords{70}{120}

\begin{tikzpicture}[tdplot_main_coords, scale=.65]
    \def\N{5}
    
    \draw[->, thick] (0,0,0) -- (4.5,0,0) node[anchor=north east] {$m$};
    \draw[->, thick] (0,0,0) -- (0,5,0) node[anchor=north west] {$t$};
    \draw[->, thick] (0,0,0) -- (0,0,4) node[anchor=south] {$d$};
    
    \foreach \y in {0,...,4} {
        \foreach \z in {0,...,3} {
            \pgfmathtruncatemacro{\cond}{(\y==3) || (\z==2)}
            \ifnum\cond=1
                \draw[gray, thin] (0, \y, \z) -- (3, \y, \z);
            \fi
        }
    }
    
    \foreach \x in {0,...,3} {
        \foreach \z in {0,...,3} {
            \pgfmathtruncatemacro{\cond}{(\x==2) || (\z==2)}
            \ifnum\cond=1
                \draw[gray, thin] (\x, 0, \z) -- (\x, 4, \z);
            \fi
        }
    }
    
    \foreach \x in {0,...,3} {
        \foreach \y in {0,...,4} {
            \pgfmathtruncatemacro{\cond}{(\x==2) || (\y==3)}
            \ifnum\cond=1
                \draw[gray, thin] (\x, \y, 0) -- (\x, \y, 3);
            \fi      
        }
    }
    
      \foreach \x in {0,...,2} {
        \foreach \y in {0,...,3} {
          \foreach \z in {0,...,2} {
            \pgfmathtruncatemacro{\cond}{(\x>=2) || (\y>=3) || (\z>=2)}
            \ifnum\cond=1
                \pgfmathsetmacro{\redcomp}{255 - 60*\x}%
                \pgfmathsetmacro{\greencomp}{255 - 50*\y}%
                \pgfmathtruncatemacro{\redcompInt}{\redcomp}%
                \pgfmathtruncatemacro{\greencompInt}{\greencomp}%
                \node[draw, circle, fill={rgb,255:red,\redcompInt; green,\greencompInt; blue,200}, inner sep=2.6pt] at (\x,\y,\z) {};%

            \fi
          }
        }
    } 
    
    \foreach \y in {0,...,4} {
        \foreach \z in {0,...,3} {
            \ifthenelse{\equal{\y}{4} \OR \equal{\z}{3}}{
                \draw[black, thick] (0, \y, \z) -- (3, \y, \z);
            }{
             }{}
            }
        }
    
    \foreach \x in {0,...,3} {
        \foreach \z in {0,...,3} {
            \ifthenelse{\equal{\x}{3} \OR \equal{\z}{3}}{
                \draw[black, thick] (\x, 0, \z) -- (\x, 4, \z);
            }{
             }{}
        }
    }
    
    \foreach \x in {0,...,3} {
        \foreach \y in {0,...,4} {
            \ifthenelse{\equal{\x}{3} \OR \equal{\y}{4}}{
                \draw[black, thick] (\x, \y, 0) -- (\x, \y, 3);
            }{
             }{}
        }
    }
    
    \definecolor{DARK_blue}{RGB}{0,0,139} %
    \definecolor{blue}{RGB}{0,0,255}       %

      \foreach \x in {0,...,3} {
        \foreach \y in {0,...,4} {
          \foreach \z in {0,...,3} {
            \pgfmathtruncatemacro{\cond}{(\x>=3) || (\y>=4) || (\z>=3)}
            \ifnum\cond=1
                \pgfmathsetmacro{\redcomp}{255 - 60*\x}%
                \pgfmathsetmacro{\greencomp}{255 - 50*\y}%
                \pgfmathtruncatemacro{\redcompInt}{\redcomp}%
                \pgfmathtruncatemacro{\greencompInt}{\greencomp}%
                \node[draw, circle, fill={rgb,255:red,\redcompInt; green,\greencompInt; blue,200}, inner sep=3pt] at (\x,\y,\z) {};%

            \fi
          }
        }
    } 
    
\end{tikzpicture}
    \end{adjustbox}
    \\
    \vskip 30pt
    \small{\captionc}
\end{minipage}%
\caption{  \label{fig:ortho_schemas}
A geometric visualization of the proposed update rule.
\captiona~A single current state vector, $\vs_{\cur} = \mathbf{S}_{\cur}[: \, , i]$, an incoming token representation, $\vh_t$, and its component orthogonal to the current state, $\vh_t^{\perp {\vs}_{\cur}}$.
\captionb~Comparison of the updated state according to the proposed update rule
($\vs_{\nxt}= \vs_{\cur} + \alpha_{i,t} \, \vh_t^{\perp {\vs}_{\cur}}$) 
and the updated state resulting from the superposition recurrence update of the standard linear attention ($\hat{\vs}_{\nxt}= \vs_{\cur} + \alpha_{i,t} \, \vh_t$, shown with a dashed arrow).
For simplicity, a unit writing intensity ($\alpha_{i,t}=1$) is assumed in both recurrent update rules.
\captionc~Visualization of the relationships between ${d \times m}$ state matrices over time in state-dependent compression, depicted as interconnections of nodes in a 3D lattice. Each memory slot (state vector) is represented by a unique color.
}
\end{figure}

\paragraph{Summary of contributions}
In this paper, we propose \model, a novel approach designed to address quadratic complexity of the attention layers.
Our method compresses the cache into a fixed number of slots by leveraging the inherent low-rank structure of K-V matrices within an online optimization framework.
This approach allows us to derive efficient, recursive memory update rules conditioned on its existing state and the current token with sub-quadratic complexity. 
In contrast to existing SSMs/RNNs, which often rely on heuristics for memory management and lack explicit optimization for compression, we formulate the compression task as an optimization problem and use online gradient descent to drive the recurrent update rule for the memory, which results in an interpretable and expressive non-linear recurrent model. 
\model updates 
each memory slot exclusively with non-redundant information, specifically by incorporating only  the component of the input token that is  orthogonal to the current state of that memory slot.
To ensure scalability, we derive an efficient, chunk-wise parallel form for this orthogonal update rule.
Finally, we empirically demonstrate the superior memory efficiency and language modeling capabilities of \model.

\section{Background}

Given an input sequence $\mathcal{X} = [\vx_1, \dots, \vx_T]$ with $\vx_t \in \mathbb{R}^{d}$, the causal softmax attention mechanism generates output tokens $\vy_t \in \mathbb{R}^{d}$, by attending to past tokens as:
\begin{align} \label{eqn:softmax}
\vy_t = \mathcal{V}_t ~ \texttt{Softmax}(\mathcal{K}_t^\top ~ \vq_t)  ~.
\end{align}
Here, the queries, keys, and values are computed by linear projections of the input: 
$\vq_t = \mathbf{W}_q~ \vx_t$, $\vk_t = \mathbf{W}_k~ \vx_t$, $\vv_t = \mathbf{W}_v~ \vx_t$, where $\mathbf{W}_q,$ $\mathbf{W}_k$, $\mathbf{W}_v$ $\in \mathbb{R}^{d \times d}$ are learnable weight matrices.
The key-value memory, represented by the caches $\mathcal{K}_t \in \mathbb{R}^{d \times t}$ and $\mathcal{V}_t \in \mathbb{R}^{d \times t}$, 
stacks the key and value vectors of each new token, resulting in a cache size that grows linearly with time.
The retrieval of relevant information from this key-value cache can be rewritten as a weighted sum: 
$
\vy_t = \mathcal{V}_t ~ \va_t, ~~
\text{where } \va_t = \texttt{Softmax}(\mathcal{K}_t^\top ~ \vq_t ) \in \mathbb{R}^{t}
$
represents the attention scores, capturing the correlations between the current token at step $t$ and its historical context (past tokens).
Consequently, the attention mechanism in \eqref{eqn:softmax} can be interpreted as performing a non-linear query over an unbounded memory.
The linear growth of the key-value cache creates a significant memory bottleneck during inference, particularly for long sequences.
Furthermore, each retrieval operation scales linearly with sequence length, resulting in an overall quadratic computational complexity $\mathcal{O}(T^2)$ for generating a full sequence of length $T$.

To address the computational and memory bottlenecks of the Softmax attention, various alternatives have been proposed~\citep{tay2022efficient}. 
A well-established approach involves employing the kernel trick to replace the softmax operation with a dot product of feature maps, $\phi(\vq_t) and \phi(\vk_t)$,~\citep{katharopoulos2020transformers}, commonly known as \textit{linear attention} (LA). This mechanism can be formulated as:
$ 
\vy_t =  \big(\sum_{i=1}^{t}\vv_i  \phi(\vk_i)^\top \big) ~ \phi(\vq_t),
$
which can be expressed as the following linear recurrent model, also known as an input dependent state-space model (SSM)\footnote{Consistent with many linear attention models, we omit the normalization term here to avoid potential numerical instabilities~\citep{qin2022devil}. Furthermore, we employ an identity mapping as the feature map, effectively absorbing any transformation into the corresponding projection layers.
}: 
\begin{align} \label{eqn:LA}
\{\vy_t\}_{t=1}^T &= \op{LA}(\{\vq_t, \vk_t, \vv_t\}_{t=1}^T)   
:= \begin{cases}
\mathbf{S}_{\nxt} = \mathbf{S}_{\cur} + \vv_t ~ \vk_t^\top ,  & \text{\textit{recurrence}}  \\
\vy_t = \mathbf{S}_{\nxt} \vq_t 
& \text{\textit{memory read-out}}  
\end{cases}
\end{align}
This representation employs a simple linear recurrence to update the matrix-valued state $\mathbf{S}_t$, which  compactly encodes key-value associations memory at each time step. Importantly, the linearity is key to achieving sub-quadratic parallel computation during training, using methods such as chunkwise computation~\citep{hua2022transformer, kacham2024polysketchformer} or parallel scan~\citep{blelloch1990prefix, smith2023simplified}, while retaining a constant-time complexity per token during inference.

An alternative strategy to ensure bounded computational and memory requirements is to maintain a fixed-size key-value cache, where the memory matrices $ \mathbf{K}, \mathbf{V} \in \mathbb{R}^{m \times d} $ are constrained to a fixed length $ m \ll T $. A standard implementation of this idea is the sliding window attention which retains the most recent $ m $ tokens via a first-in-first-out (FIFO) queue.
While computationally efficient, sliding window attention suffers from a limited receptive field. This restricts the model's ability to capture long-range dependencies and maintain global context, resulting in a poor recall-memory trade-off~\citep{arora2024simple}.
On the other hand, a growing body of research has observed that the key-value matrices in the attention layers often exhibit structured low-rank properties~\citep{wang2020linformer, chen2021scatterbrain, singhania2024loki}.  This insight suggests that instead of naively truncating memory, we can develop \textit{efficient compression} techniques that selectively distill and store the essential context while discarding less relevant or redundant information.

The update rule in the linear attention, and its gated variant, typically relies on an additive outer product of input-dependent representations, which can be generally expressed as:
$
\mathbf{S}_{\nxt} = \mathbf{S}_{\cur} + f_{g}(\vx_t) \otimes f_v(\vx_t)
$
where $f_v(\vx_t) $ is an embedding of the input token and  $f_{g}$ can be interpreted as an input gate that controls the writing intensity.\footnote{These two are also called \textit{role} and \textit{filler} vectors in tensor product representation~\citep{smolensky1990tensor}.}
While this \textit{linear rank-one modification} to the state matrix (\textit{a.k.a.} Hebbian-like update rule~\citep{hebb2005organization}) enables efficient parallel computation, it suffers from a key limitation:  the additive update term in the recurrence is agnostic to the current  memory state $\mathbf{S}_{\cur}$ and 
operates independently of it. 
This lack of state awareness can cause \textit{key interference} and eventually lead to an \emph{overcapacity regime}~\citep{schlag2021linear},  where multiple tokens attempt to write to the same memory slot when the memory size is smaller than the sequence length.

Based on this insight, ideally, the writing intensity of the $t$-th token $\vx_t$ to the $j$-th memory slot, $(\mathbf{S})_{_{j,:}}$, should depend on the interaction between the new token itself and the content of that slot. 
From a gating perspective, the gating mechanism should have access to the current state of the memory to make informed decisions about which information to add or discard~\citep{hochreiter1997LSTM, gers2002peepholeLSSTM}.
This requires a \emph{state-dependent gating} mechanism that dynamically modulates updates based on the current memory state. 
Although a naive conditioning on the state can break the ability to parallelize computations, and hence it has been avoided in previous works, we address this challenge through chunk-wise approximations later in the manuscript.
In the following section, we frame the role of the recurrent layer as solving an online optimization problem and derive an optimal update rule to compress and retain essential information from a sequence.

\section{Compression Layer}

\paragraph{State-Dependent Compression for Unbounded Caches}

Our objective is to develop a compression model that dynamically updates and maintains a compact representation of the contextual history—conventionally stored in the key and value caches of a transformer model—in a streaming manner.  As new tokens arrive, the model selectively distills and stores essential contextual information into a compressed memory matrix. 
This enables computationally efficient querying, as the memory read-out is processed using the compressed state, i.e.,
$ \vy_t = \mathbf{S}_t \hat{\vr}_t$ instead of querying the full cache. 
Here, $\hat{\vr}_t$ represents a retrieval vector analogous to the attention weights in standard attention layers.

This lossy compression approach entails a trade-off between computational efficiency, memory usage, and query precision. A more compact memory representation (i.e., smaller $m$) reduces both computational cost and memory footprint, but at the expense of information loss and lower fidelity in reconstructing the original context, thereby diminishing the overall expressivity of the model.
We aim to design an optimal lossy compression layer that minimizes this precision loss.
We formulate this problem as an input reconstruction task, where we enforce:
$\mathbf{x}_t \approx \tilde{\mathbf{x}}_t =  \mathbf{S}_t \vrep_t$.
Here, $\tilde{\mathbf{x}}_t$ is the reconstructed input, $\mathbf{S}_t$ represents the dynamically updated state,
and $\boldsymbol{\rep}_t$ is a latent representation vector\footnote{Note that while the objective reconstructs from $\mathbf{S}_t$ only the current time step $\vx_t$, due to the online gradient descent form of how this objective is minimized, the state $\mathbf{S}_t$ implicitly compresses the history, allowing it to reconstruct the entire sequence, i.e. $\mathbf{S}_t \vrep_k \approx \vx_k$.
Since $\mathbf{S}_t$ has the same interpretation as the state in RNNs and SMMs, the same notation is reused.}.

Inspired by classical representation learning techniques such as dictionary learning, sparse coding, and structured matrix factorization~\citep{mairal2009online, lyu2020online}\footnote{This problem has been studied under various names over the decades, including dictionary learning, factor analysis, topic modeling, and component analysis, each with slightly different constraints and emphases~\citep{lyu2020online}.}, 
we interpret our approach as dynamically learning and updating basis vectors (a.k.a. dictionary atoms) and their corresponding latent coefficients (analogous to sparse codes).

\subsection{Decoding Layer} \label{sec:dec}
For each input sequence, we model a \emph{decoding layer}, denoted as $g(\vk_t; \mathbf{S}_t)$, which operates on the latent representation $\vrep_t$ and is parameterized by the state matrix $\mathbf{S}_t$.
Unlike standard neural network layers, here we aim to dynamically update $\mathbf{S}_t$, over the course of a sequence, thereby effectively memorizing and encoding the historic context up to time $t$.
Functionally, this acts as a decoding layer with an \textit{internal state}, or equivalently, a \textit{fast decoding layer}.
Specifically, each token embedding ${\vv}_t$ is paired with its corresponding latent representation (code) ${\vrep}_t$,
and the decoding function $g({\vrep}_t; \mathbf{S}_t)$ aims to reconstruct ${\vv}_t$. 
To achieve this, we formulate an optimization problem that minimizes a loss function $\ell$—quantifying the dissimilarity between the decoded output and the target vector—as its objective at each time step:
\begin{align}
\mathcal{L}_t = \ell \big( g({\vrep}_t; \mathbf{S}_t),~  {\vv}_t  \big), ~~   \mathbf{S}_t \in \mathbb{R}^{d \times m}, ~ \vv_t \in \mathbb{R}^{d}, ~ \vrep_t \in \mathbb{R}^{m}
\label{eq:decoding_loss}
\end{align}
We refer to this objective as the \textit{compression loss} throughout this paper. 
Here, the latent representation $\vrep_t$ is generated by a model-based encoder network, modeled  simply as a linear projection of the input:
${\vrep}_t = \mathbf{W}_{\rep} \mathbf{x}_t $
where $\mathbf{W}_{\rep} \in \mathbb{R}^{m \times d_x}$ is a projection weight matrix. 
This weight remains fixed during the internal state updates and is trained jointly with the rest of the model parameters in the outer training loop.
This setup aligns with established meta-learning frameworks~\citep{schmidhuber1992learning,thrun1998learning,andrychowicz2016learning,sun2024ttt} or bilevel optimization approaches~\citep{liu2022bome, chen2022gradient}.

The proposed framework consists of two distinct types of parameters: (I) the internal states of the compression layers, $\mathbf{S}_t$, which dynamically store in-context information for each sequence, and (II) the outer model parameters, including the projection layer weights, collectively denoted  as $\boldsymbol{\mathcal{W}}$, which capture broader patterns across the training set. This leads to a bilevel learning process composed of:
\begin{itemize}
\item \emph{Inner Loop (State Update):}  A fast update mechanism that adapts the internal states $\mathbf{S}_t$ for each token within a sequence by minimizing the compression loss in \eqref{eq:decoding_loss}. Each sequence effectively serves as a distinct dataset for the inner loop, which encodes in-context information into a sequence of evolving states  
$\{\mathbf{S}_t\}_{t=1}^T$.
Throughout this process, the outer model weights,
$\boldsymbol{\mathcal{W}}$, remain frozen.
\item  \emph{Outer Training Loop:} The standard training of the neural network that learns $\boldsymbol{\mathcal{W}}$ by minimizing the average loss across all training sequences for the (self-)supervised learning task.   
This slower loop typically employs standard optimizers such as ADAM~\citep{kingma2014adam} to learn generalizable patterns from the training dataset.  
\end{itemize}
From an optimization perspective, this bilevel process is analogous to alternating optimization~\citep{goldstein2014fast}, where the inner loop optimizes the state $\mathbf{S}_t$ while keeping $\boldsymbol{\mathcal{W}}$ fixed, and the outer loop optimizes $\boldsymbol{\mathcal{W}}$ based on the adapted states.

The focus of this work is on designing an optimal update rule for the memory states.
Due to its streaming nature, a standard approach for a sequence model is to treat the inner loop as an online regression problem and employ steepest descent. Specifically, the internal state is dynamically updated using a single gradient descent step per token:
\begin{align} \label{eqn:OGD}
    \mathbf{S}_{\nxt} = \mathbf{S}_{\cur} - \gamma_t \nabla_{S} \mathcal{L}(\mathbf{S}_{\cur}, \vv_t, \vrep_t) 
\end{align}
This recursive update yields a sequence of states $\{\mathbf{S}_t\}_{t=1}^T$, where each new state $\mathbf{S}_{\nxt}$ is a nonlinear function of the current state and the current input tokens, ensuring a causal and context-dependent evolution of the internal state.

\subsubsection{State Normalization} \label{sec:state_norm}
Following common practices in dictionary and subspace learning—where basis vectors (a.k.a. dictionary atoms or principal components) are normalized—we apply column-wise normalization to each state vector of the state matrix.
Hence, the decoding function is defined using the normalized state matrix $\phi(\mathbf{S}_t )$ as:
\[\hat{\vv_t} = g({\vrep}_t; \mathbf{S}_t) = \phi(\mathbf{S}_t ){\vrep}_t. \]
At each time step, the internal states are updated to ensure the linear combination of normalized state vectors closely approximates the target vector $\vv_t$. We explore two objectives to achieve this.

First, we adopt the standard $\ell_2$ reconstruction loss, which measures the squared Euclidean distance between the decoded vector and the target:
\begin{align}
\mathcal{L}_t = \| \phi(\mathbf{S}_t) {\vrep}_t -  {\vv}_t \|^2, ~~   \mathbf{S}_t \in \mathbb{R}^{d \times m}, ~ \vv_t \in \mathbb{R}^{d}, ~ \vrep_t \in \mathbb{R}^{m}
\label{eq:decoding_stateNorm_loss}
\end{align}
To derive the closed-form gradient of this objective, let's define the normalized state matrix:
$ \Phi = [\vphi_1, ..., \vphi_m] $, where $ \vphi_i = \frac{\vs_i}{\|\vs_i\|} $ and $\vs_i$ is the $i$-th column of $\mathbf{S}_{\cur}$ ($i$-th basis vector), and denote the reconstruction error as $ \ve_t := \phi(\mathbf{S}_{\cur}) {\vrep}_t- {\vv_t}$.
The Jacobian of the normalization function is
$$
\mathbf{J}_\phi (\vs_i)  =  \frac{1}{\|\vs_i\|}
\left(\mathbf{I} - \frac{\vs_i \vs_i^\top}{\|\vs_i\|^2}\right)
= \frac{\mathbf{P}({\vs}_i)}{\|\vs_i\|}.
$$
Therefore, by applying the chain rule, the gradient of the loss with respect to $ \mathbf{S} $ is given by:
\begin{align} \label{eq:decoding_recurrence}
\nabla_{\mathbf{S}} \mathcal{L}_t  & =
\begin{bmatrix} 
 \ve_t^\top \frac{\mathbf{P}({\vs}_1)}{\|\mathbf{s}_1\|} \rep_{t_1},  
 & \dots, &   
  \ve_t^\top \frac{\mathbf{P}({\vs}_m)}{\|\mathbf{s}_m\|} \rep_{t_m} 
\end{bmatrix} 
=  (\ve_t^\top \times_1 \mathcal{P}) \odot \vrep_t^\top 
\end{align}
where $\odot$ denotes the element-wise (Hadamard) product with broadcasting, and $\times_1$ is a vector-tensor product defined as
$\ve^\top \times_1 \mathcal{P} := \begin{bmatrix} \ve^\top \mathcal{P}_{:,:,1} ,~ \dots, \ve^\top \mathcal{P}_{:,:,m} \end{bmatrix}
$\footnote{This can be implemented using Einstein summation as \texttt{einsum("d1, d1 d m -> d m", e, J)}. 
},
while the tensor 
$\mathcal{P} := 
\begin{bmatrix} 
\frac{\mathbf{P}({\vs}_1)}{\|\mathbf{s}_1\|}, \dots,\frac{\mathbf{P}({\vs}_m)}{\|\mathbf{s}_m\|} 
\end{bmatrix}
\in \mathbb{R}^{d \times d \times m}$ is formed by stacking Jacobian matrices along the last dimension.

Alternatively, the compression objective can be formulated to maximize the dot-product similarity between the decoded output and target representation: 
\begin{align}
\mathcal{L}_t = - \langle \phi(\mathbf{S}_t) \vk_t ,~ \vv_t \rangle, ~~   \mathbf{S}_t \in \mathbb{R}^{d \times m}, ~ \vv_t \in \mathbb{R}^{d}, ~ \vrep_t \in \mathbb{R}^{m}
\label{eq:similarity_obj_stateNorm}
\end{align}
The closed-form expression for the gradient of this loss can be derived similarly: 
\begin{align}
  \nabla_{\mathbf{S}} \mathcal{L}_t   = - \vv_t^\top 
  \times_1
  \begin{bmatrix} 
    \frac{\mathbf{P}({\vs}_1)}{\|\mathbf{s}_1\|}, \dots,     \frac{\mathbf{P}({\vs}_m)}{\|\mathbf{s}_m\|} 
\end{bmatrix} \odot \vrep_t^\top
\label{eq:similarity_recurrence}
\end{align}

\begin{highlightbox}
This gradient derivation reveals a \emph{highly interpretable and interesting update rule}.
The matrix
$\mathbf{P}({\vs}_i) = \mathbf{P}({\vphi}_i) := \left( 
\mathbf{I} -
\frac{\vs_{i} {\vs}_{i}^\top}{\|\vs_{i}\|^2}  \right)  $,  which appears in ~\eq{eq:decoding_recurrence} and \eq{eq:similarity_recurrence},
is formally known as the \textit{projection matrix onto the orthogonal complement} of $ {\vs}_i$ in linear algebra~\citep[\S 3.3]{strang2000linearAlgebra}.
This insight implies that the update for each memory slot (column $ \vs_i$) is driven by a projection of the input vector (e.g., $\vh_t \! = \! -\vv_t$ in \eq{eq:similarity_recurrence} or $\vh_t \!= \!\ve_t$ in \eq{eq:decoding_recurrence}) onto the space orthogonal to that slot. 
This suggests an interpretable decomposition of the $\vh_t$ into two components:
I) $\vh_t^{\perp \vs_i}$, \emph{the component orthogonal to $\vs_i$}, which is \emph{used} to update the memory slot.
II) $\vh_t^{\| \vs_i}$, {the component of $\vh_i$ aligned} with $\vs_i$, which is \emph{discarded} in the update rule, ensuring non-redundant updates.
This implies that \emph{each memory slot is updated only with new information that is not already captured in that slot}. 
The scalar $\rep_{i,t}$ acts as a \emph{writing intensity}, determining the t-th token's contribution to the $i$-th memory slot.
This orthogonal update process is visualized in \Figref{fig:ortho_schemas}.
\end{highlightbox}

Therefore, applying online gradient descent (OGD)~\eq{eqn:OGD} to the compression losses offers a principled approach for deriving a recurrent update rule based on an orthogonal projection onto the current state.
We unify these update rules into a general formulation, referred to as \emph{Orthogonal State Recurrence (OSR)}:
\begin{align} \label{eq:OSR}
\mathbf{S}_{\nxt} = \mathbf{S}_{\cur} - \gamma_t \vh_t^\top
    \times_1
    \begin{bmatrix} 
        \frac{\mathbf{P}({\vs}_1)}{\|\vs_1\|}, \dots,     \frac{\mathbf{P}({\vs}_m)}{\|\vs_m\|} 
    \end{bmatrix} \odot \vrep_t^\top  
\end{align}
Here, $\vh_t := \ve_t$ for the $\ell_2$ loss and $\vh_t := -\vv_t$ for the dot-product similarity objective.
Table \ref{tbl:ogd-list}  compares the proposed OSR updates of the compression layers with the OGD-based recurrences of existing RNNs.
An alternative encoding representation is detailed in Appendix \ref{sec:encoding}, while a simplified form of the update is provided in Section \ref{sec:complexity}.

\subsection{Stabilizing Memory Updates via Normalization}

At each recurrence step, a memory slot is updated by incorporating only the component of the new information that is \textit{orthogonal} to its current state. Formally, we update the $i$-th memory slot as
$
\vs_{i,\nxt} \text{ = } \vs_{i,\cur} \text{ + } \Delta \vs_{i,t},
$
where
$
\Delta \vs_{i,t} := c_{i,t} \, \vh_t^{\perp \vs_{i,\cur}}
$
, with $\vh_t^{\perp \vs_{i,\cur}}$ denoting the orthogonal component of the input relative to $\vs_{i,\cur}$ and $c_{i,t}$ representing an input-dependent writing intensity. 
While this update scheme avoids interfering with the existing memory by \emph{adding only novel, non-redundant information}, it inherently leads to a monotonic increase in the norm of $\vs_i$ with each update, as dictated by the Pythagorean theorem:
$
\|\vs_{i,\nxt}\|^2 = \|\vs_{i,\cur}\|^2 + \|\Delta \vs_{i,t}\|^2
$.
This unbounded growth can induce numerical instability and state magnitude explosion or may dilute the effective representation of information over time.

To address this issue, we constrain the feasible set for the state vectors to the unit sphere 
$\mathcal{C} = \{ \vs \in \R^d \mid \|\vs\| = 1 \}$, and enforce this constraint by projecting the resulting Euclidean update back onto $\mathcal{C}$, denoted by $\mathcal{P}_{\mathcal{C}}(\cdot)$, at each time step.
Therefore, the effective update becomes  
\begin{align}     \label{eqn:normalized_recurrence}
    \vs_{i,\nxt} = \mathcal{P}_{\mathcal{C}}(\vs_{i,\cur} + \Delta \vs_{i,t}) = \beta_{i,t} \left(\vs_{i,\cur} + \Delta \vs_{i,t} \right),
    \\  
    \text{where }
    \beta_{i,t} = \left({1+ \|\Delta \mathbf{s}_{i,t}\|^2} \right)^{-\frac{1}{2}}, 
    \text{ assuming } \|\mathbf{s}_{i,\cur}\|=1. 
    \nonumber
\end{align}
This normalization, achieved by scaling with $\beta_{i,t}$, ensures that the updated state $\vs_{i,\nxt}$ remains on the unit sphere  while preserving the steepest-descent direction, thereby maintaining stability and allowing the model to effectively store relevant information.
Functionally, this normalization of the recurrence terms acts analogously to a forgetting gate in RNNs and adaptively normalizes the effective step size of the update term—a technique known to improve convergence in optimization algorithms like Adagrad~\citep{duchi2011adaptive} and Adam~\citep{kingma2014adam}.
We formalize the relationship between the proposed 
\emph{Normalized Orthogonal State Recurrence} (NOSR)
and Riemannian optimization~\citep{absil2009optimization, boumal2023introduction}, in the following proposition.

\begin{proposition}[Equivalence to Gradient Descent on Riemannian Manifold]
\label{thm:Riemannian}
Let $\mathcal{C} = \{ \mathbf{s} \in \mathbb{R}^d \mid \|\mathbf{s}\| = 1 \}$ denote the unit sphere. 
Then, the projected gradient update of the form 
$\vs_{i,\nxt} = \mathcal{P}_{\mathcal{C}}(\vs_{i,\cur} + \Delta \vs_{i,t})$ (as in \eqref{eqn:normalized_recurrence}),
where the update term $\Delta \vs_{i,t}$ lies in the subspace orthogonal to $\vs_{i,\cur}$ (cf. \eq{eq:OSR}), is equivalent to a retraction step in Riemannian optimization~\citep{bonnabel2013stochastic}.
\end{proposition}

\begin{table}[t]
\scriptsize
\caption{ \label{tbl:ogd-list}
\footnotesize
Comparison of the objective functions and their corresponding online gradient descent updates for the proposed and existing RNNs.
We include several linear RNNs for comparison:
\method{ Linear-Attention (LA)}~\citep{katharopoulos2020transformers}, \method{Mamba2}~\citep{mamba2} and \method{DeltaNet}~\citep{schlag2021linear, yang2024parallelizing},
\method{GLA}~\citep{yang2024gatedattn}, 
\method{RWKV-6/7}~\citep{peng2024Rwkv6, peng2025rwkv}, 
\method{Gated-DeltaNet}~\citep{yang2024gatedDeltaNet},
and \method{TTT}~\citep{sun2024ttt}.
It is worth noting that, the effective recurrent update of the compression layers after re-scaling becomes online gradient descent on Riemannian manifold: 
$\displaystyle \mathbf{S}_{\nxt} = \mathbf{1} \, \vbeta_t^\top \odot \big( \mathbf{S}_{\cur} + \Delta \mathbf{S}_{t}  \big)  (\eqref{eqn:normalized_recurrence})$. 
\textsc{LA} can be interpreted as online gradient descent with a fixed step size ($\gamma_t=1$); however, more flexible, input-dependent step sizes are frequently used in recent RNNs~\citep{orvieto2023LRU,qin2024hgrn2,gu2023mamba}. 
\method{Mamba2} and \method{Gated-DeltaNet} employ a forgetting gate, which is equivalent to performing online gradient descent with L2 regularization and regularization factor $\lambda_t$.
In \method{Mamba2}, the forget gate is controlled by $\mu_t = 1-\lambda_t$, and
the reparameterization for the forget gate and step size of \method{Gated-DeltaNet} is discussed in~\citep{wang2025testTimeRegression}.
In \method{RWKV-7}, the regularizer is column-wise: 
$\frac{1}{2} \|\rmS_t^\top\|^2_{\op{diag}(\vlambda_t)} := \frac{1}{2} \op{Tr}(\rmS_t^\top \op{diag}(\vlambda_t) \rmS_t)$, resulting in a diagonal-plus-low-rank transition matrix.
Here, $\times_1$ denotes vector-tensor product defined as
$\ve^\top \times_1 \begin{bmatrix} \mathbf{J}_1 ,~ \dots, \mathbf{J}_m \end{bmatrix}  
= 
\begin{bmatrix} \ve^\top \mathbf{J}_1 ,~ \dots, \ve^\top \mathbf{J}_m \end{bmatrix}
$.
}
\setlength{\tabcolsep}{2.5pt}
\renewcommand{\arraystretch}{1.8}
\begin{tabular*}{\textwidth}{@{\quad}p{0.2\textwidth}@{\quad\quad\quad}p{0.33\textwidth}@{\quad\quad}p{0.5\textwidth}@{}}
\toprule
\textbf{Method} & \textbf{Objective} $\mathcal{L}_t$ & \textbf{Online Gradient Descent Update} \\
\midrule
\method{Linear-Attention} %
& $\displaystyle  -\langle\rmS_t \vk_t, \vv_t\rangle$ & $\displaystyle \mathbf{S}_{\nxt} = \mathbf{S}_{\cur} + \vv_t \vk_t^\top$ \\
\method{Mamba2} %
& $\displaystyle  -\langle\rmS_t \vk_t, \vv_t\rangle + \frac{\lambda_t}{2} \|\rmS_t\|_2^2$ 
& $\displaystyle \mathbf{S}_{\nxt} = \mu_t \mathbf{S}_{\cur} + \vv_t \vk_t^\top$ \\
\method{{RWKV-6 ~ \& ~ GLA}} 
& $\displaystyle  -\langle\rmS_t \vk_t, \vv_t\rangle + \frac{\vlambda_t}{2} \|\rmS_t\|_2^2$ 
& $\displaystyle \mathbf{S}_{\nxt} = \op{diag}(\vmu_t) \, \mathbf{S}_{\cur} + \vv_t \vk_t^\top$ \\
\method{DeltaNet} %
& $\displaystyle  \Big\|\mathbf{S}_t \vrep_t - \vv_t \Big\|^2 $ & $\displaystyle \mathbf{S}_{\nxt} = \mathbf{S}_{\cur}(\rmI - \gamma_t \vk_t \vk_t^T) + \gamma_t \vv_t \vk_t^T$ \\
\method{Gated-DeltaNet} %
& $\displaystyle  \Big\|\mathbf{S}_t \vrep_t - \vv_t \Big\|^2 + \frac{\lambda_t}{2} \|\rmS_t\|_2^2$ 
& $\displaystyle \mathbf{S}_{\nxt} = \mu_t \mathbf{S}_{\cur} (\rmI - \gamma_t \vk_t \vk_t^T) + \gamma_t \vv_t \vk_t^T$ \\
\method{{RWKV-7}} 
& $\displaystyle  \Big\|\mathbf{S}_t \vrep_t - \vv_t \Big\|^2 + 
\frac{1}{2} \|\rmS_t^\top\|^2_{\op{diag}(\vlambda_t)}$
& $\displaystyle \mathbf{S}_{\nxt} = \mathbf{S}_{\cur} (\op{diag}({\vmu_t}) - \gamma_t \vk_t \vk_t^T) + \gamma_t \vv_t \vk_t^T$ \\
\method{TTT} %
& $\displaystyle  \Big\| \phi(\mathbf{S}_t {\vrep}_t ) -  {\vv}_t \Big\| ^2 $ 
& $\displaystyle \mathbf{S}_{\nxt} = 
\mathbf{S}_{\cur} - \gamma_t \ve_t^\top \, \frac{\mathbf{P}({\vz}_t)}{\|{\vz}_t\|}
\vrep_t^\top
$ \\
\method{\model~(Dec)}~(\ref{eq:decoding_stateNorm_loss})  
& $\displaystyle  \Big\| \phi(\mathbf{S}_t) {\vrep}_t  -  {\vv}_t \Big\|^2 $ 
& 
$\displaystyle \mathbf{S}_{\nxt} =  \mathbf{1} \, \vbeta_t^\top \!\odot\! 
\big(
\mathbf{S}_{\cur} - \gamma_t \ve_t^\top
\!\times_1\!
\begin{bmatrix} 
    \frac{\mathbf{P}({\vs}_1)}{\|\mathbf{s}_1\|}, \dots, \frac{\mathbf{P}({\vs}_m)}{\|\mathbf{s}_m\|} 
\end{bmatrix} \!\odot\! \vrep_t^\top
\big)
$ \\
\method{\model~(Sim)}~(\ref{eq:similarity_obj_stateNorm})
& $\displaystyle  - \langle \phi(\mathbf{S}_t)^\top \vv_t ,~ \vk_t \rangle$ 
& $\displaystyle \mathbf{S}_{\nxt} = 
\mathbf{1} \, \vbeta_t^\top \!\odot\! 
\big(
\mathbf{S}_{\cur} + \gamma_t \vv_t^\top 
\!\times_1\!
\begin{bmatrix} 
    \frac{\mathbf{P}({\vs}_1)}{\|\mathbf{s}_1\|}, \!\dots\!,     \frac{\mathbf{P}({\vs}_m)}{\|\mathbf{s}_m\|} 
\end{bmatrix} \odot \vrep_t^\top
\big)
$ \\
\method{\model~(Enc)}~(\ref{eq:encoding_stateNorm_loss})  
& $\displaystyle  \Big\| \phi(\mathbf{S}_t)^\top\, \vv_t - \vk_t \Big\|^2 $ 
& $\displaystyle \mathbf{S}_{\nxt} = 
\mathbf{1} \, \vbeta_t^\top \!\odot\! 
\big(
\mathbf{S}_{\cur} - \gamma_t \vv_t^\top
\!\times_1\!
\begin{bmatrix} 
   \frac{\mathbf{P}({\vs}_1)}{\|\mathbf{s}_1\|}, \!\dots\!,     \frac{\mathbf{P}({\vs}_m)}{\|\mathbf{s}_m\|} 
\end{bmatrix} \odot \ve_t^\top  
\big)
$ \\
\bottomrule
\end{tabular*}
\end{table}

\subsection{Forgetting by State Regularization}

Similar to how regularization is used in standard neural network training to control the memorization of the model, we can apply regularization to the states in the inner loop to manage memory retention. 
Specifically, applying $\ell_2$ regularization to the state matrix $\mathbf{S}_t$ yields the regularized objective function:
$
\hat{\mathcal{L}}_t = \| \phi(\mathbf{S}_t) \, {\vrep}_t -  {\vv}_t \|^2
+ \frac{\lambda_t}{2} \| \mathbf{S}_t \|_F^2, 
$
where $\lambda_t$ is the regularization parameter and $\| \cdot \|_F$ denotes the Frobenius norm. 
Optimizing this differentiable objective using gradient descent results in a  recurrence with state decay for the decoding compression layer (\eqref{eq:decoding_stateNorm_loss}): 
\begin{align} \label{eq:dacay_recurrence}
    \mathbf{S}_{\nxt} = 
    \mu_t \mathbf{S}_{\cur} - \gamma_t \nabla_{S} \mathcal{L}(\mathbf{S}_{\cur}, \vv_t, \vrep_t), %
\end{align}
where the scalar  $ \mu_t = 1 - \gamma_t \lambda_t \in [0, ~1]$ acts as a forget gate, controlling the proportion of the past memory that is retained in the update.

Alternatively, we can induce sparsity in the memory states by applying element-wise $\ell_1$ norm~\footnote{The $\ell_1$ norm is a relaxed version of the hard sparsity constrain, which drives small states toward zero.},  resulting in the objective function:  
\begin{align}  
\hat{\mathcal{L}}_t = \| \phi(\mathbf{S}_t) \, {\vrep}_t - {\vv}_t \|^2  
+ \lambda_t \|\mathbf{S}_t\|_1.  
\end{align}  
This non-differentiable composite objective can be efficiently optimized using the \textit{Proximal Gradient Descent Algorithm}, which iteratively performs a gradient descent with the smooth component and
then applies the proximal operator associated with the non-differentiable regularizer~\citep{parikh2014proximal}.
This iterative procedure, commonly known as
the Iterative Shrinkage-Thresholding Algorithm (ISTA)~\citep{parikh2014proximal, beck2009fast}, gives the update rule:
\begin{align} \label{eq:ista_recurrence}
\mathbf{S}_{\nxt} = 
    \op{prox}_{\gamma_t \lambda_t \|\cdot \|_1} \left( \mathbf{S}_{\cur} - \gamma_t \nabla_{S} \| \phi(\mathbf{S}_t) \, {\vrep}_t - {\vv}_t \|^2 \right),
\end{align}
where the proximal operator for the $\ell_1$ norm corresponds to the \emph{shrinkage (soft thresholding) operation}, defined as: 
$
\op{prox}_{\mu_t \| \cdot \|_1} (x)= \op{sign}(x) \max( |x| - \mu_t, 0 )
$.
By suppressing small values, this recurrence promotes sparsity in the learned memory representations. 
In the following, we focus on the $\ell_2$ regularization due to its scaling weight decay form in the recurrence \eq{eq:dacay_recurrence}.

\paragraph{General Form.} By integrating all components, we arrive at the complete form for the non-linear state transition and memory read-out of \model:
\begin{align} %
\op{\model} (\{\vk_t, \vv_t, \vq_t\}_{t=1}^T) 
\!=\!   
\begin{cases}
    \mathbf{S}_{\nxt} \! = \! (\mu_t  \vone  \vbeta_t^\top) \! \odot \!\mathbf{S}_{\cur} \!-\! \gamma_t \vh_t^\top
    \!\! \times_1 \!\!
    \begin{bmatrix}
        \frac{\mathbf{P}({\vs}_1)}{\|\vs_1\|}, \!\dots\!,     \frac{\mathbf{P}({\vs}_m)}{\|\vs_m\|} 
    \end{bmatrix} 
    \!\! \odot \!  
    (\vrep_t \! \odot \! \vbeta_t)^\top  
    \\
    \vy_t \!=\! \mathbf{S}_{\nxt} {\vq}_{t}   
\end{cases} 
\nonumber
\end{align}
where $\vbeta_t~\in \mathbb{R}^m$ is the per-slot normalization factor.
We compare the proposed orthogonal state update with the delta rule linear recurrence~\citep{widrow1988adaptive} and nonlinear update rule of $\op{TTT}$~\citep{sun2024ttt} in Appendix~\ref{apdx:proofs}.
In the following we present efficient computation of the \model and a chunk-wise parallelization.

\vspace{-5pt}
\subsection{Efficient Computation} \label{sec:complexity}
\vspace{-5pt}
The update rules presented in this work (\eqref{eq:SDC_general_stateNorm}) involves computing the projection $\vh_t^{\perp \vs_{i}} = \mathbf{P}\vh_t$. 
Given the identity plus rank-one form of $\mathbf{P}$,
the projection operation reduces to a dot product and a scalar-vector multiplication:
$
\vh_t^{\perp \vs_{i}} = \vh_t - \frac{\vs_{i} (\vs_{i}^\top \vh_t)}{\|\vs_{i}\|^2},
$
avoiding a full matrix-vector multiplication.
The computational cost of each projection is linear in 
the state dimension $d$,
leading to an overall complexity of $\mathcal{O}(d\,m)$ for the full recurrence in \eqref{eq:OSR}.
By eliminating the need for vector-Jacobian-product (\texttt{vjp}) computations,
this explicit form can be expressed as follows\footnote{For notation brevity and due to the normalization step in \eqref{eqn:normalized_recurrence}, we assume $\|\mathbf{s}_{i,\cur}\|=1 $.
}: 
\begin{align} \label{eq:recurrence_detailed}
\mathbf{S}_{\nxt}  &=  
    \mathbf{G}_t \! \odot \!\mathbf{S}_{\cur} 
    +\underbrace{\vone (\hat{\vh}_t \! \odot \! \hat{\vk}_t)^\top \! \odot \!\mathbf{S}_{\cur}
    - \vh_t \hat{\vk}_t^\top
    }_{-\gamma_t \nabla_{S} \mathcal{L}(\mathbf{S}_{\cur}, \vv_t, \vrep_t)
    } \\
\text{where } & 
\mathbf{G}_t = \mu_t  \vone  \vbeta_t^\top \in \mathbb{R}^{d \times m},~~
\hat{\vh}_t = \mathbf{S}_{\cur}^\top \vh_t \in \mathbb{R}^{m} ,~~
\hat{\vk}_t = \gamma_t \vbeta_t \! \odot \! \vrep_t \in \mathbb{R}^{m}
\nonumber
\end{align}

\subsubsection{Parallel and Hardware Efficient Form} \label{sec:parallel}
Various methods have been explored to enable parallel evaluation of non-linear RNNs. One strategy, as proposed by \citet{lim2023parallelizing,gonzalez2024towards}, involves casting inference as finding the solution to a fixed-point equation, thereby achieving parallelism.
 In a different approach, \citet{sun2024ttt} introduced a parallel chunk-wise solution using a gradient approximation. This method splits a sequence into non-overlapping chunks and utilizes the state at the beginning of each chunk to approximate the gradients for the entire chunk in parallel. 
We follow this approach and denote the state at the beginning of the chunk (i.e., the final state from the preceding chunk) by $\mathbf{S}_{t'}$, where $t' \!=\! t - \op{mod}(t, C)$ with $C$ denoting the chunk size. The gradient is then approximated as 
$\nabla_{S} \mathcal{L}(\mathbf{S}_{t'}, \vv_t, \vrep_t)$. This approximation linearizes the nonlinear recurrence \eq{eq:recurrence_detailed} as:
\begin{align} \label{eq:recurrence_linear1}
\mathbf{S}_{\nxt}  &=  
    \mathbf{G}_t \! \odot \!\mathbf{S}_{\cur} 
    + \vone (\tilde{\vh}_t \! \odot \! \hat{\vk}_t)^\top \! \odot \!\mathbf{S}_{t'}
    - \vh_t \hat{\vk}_t^\top,
    ~ \text{ where } \tilde{\vh}_t \!=\! \mathbf{S}_{t'}^\top \vh_t.
\end{align}
Now, for the time steps $t\!=\! b C \!+\! \tau$ in the $b$-th block, let $\mathbf{X}^b = \vx[bC \!+\! 1 \!:\! b(C \!+\! 1)]$ denote the stacked input into the chunk-wise matrices and $ \mathbf{X}_{\tau}^b \!=\! \vx[bC+\tau]$
(similarly for other vectors such as $\vq, \vh, \vk$),
and define
$\va_\tau^b \!=\! \prod_{i = b C \!+\! 1}^{b C \!+\! \tau} \vbeta_i \mu_i$ and 
a block lower triangular tensor $\MOmega^b \in \mathbb{R}^{C \times C \times m}$ with components $\MOmega^b_{j,i,:} = \frac{\va_j^b}{\va_i^b}  \mathbb{I}_{i \le j}$ ~.
Therefore, the layer update at the chunk-level  is expressed as:
\begin{align} \label{eq:chunkwise_OSR}
\mathbf{S}_{b}  &=  
    \big(\vone ({\va_C^b} \!+\! {\vf^b})^\top \big)\! \odot \! \mathbf{S}_{b-1}
    - {\mH^b}^\top (\hat{\mK}^b \! \odot \! \MOmega^b_{C,:,:})) \\
    \mY^b &= \big(\mQ^b \! \odot \! (\mLambda^b \!+\! \mF^b) \big) \mathbf{S}_{b-1}^\top
    - \mP^b \mH^b  \nonumber
\end{align}
where $\vf^b \!=\! \op{diag}[\tilde{\mH}^b \, (\hat{\mK}^b \! \odot \! \MOmega^b_{C::})]$ and 
$\mF^b_{i n} \!=\! \sum_{\tau=1}^{C} (\tilde{\mH}^b \! \odot \! \hat{\mK}^b)_{\tau n} \, \MOmega^b_{i \tau n}$ ~ 
(refer to Appendix \ref{apdx:details} for detailed \texttt{matmul} computation of $\vf^b, ~ \mF^b$ and $\vbeta^b_\tau$).
Here, $\mP^b \in \mathbb{R}^{C \times C}$ is a lower triangular matrix with $\mP^b_{ij} = \sum_{k=1}^{m} \mQ^b_{ik} \, \hat{\mK}^b_{jk} \, \MOmega^b_{ijk}$. 
A \texttt{matmul}-optimal computation for $\mP^b$ is presented in \cite{zhang2024gated} using a sub-tiling technique.   
Furthermore, a closer inspection of \eq{eq:recurrence_detailed} reveals that the recurrence can be simplified by linearizing only the $\hat{\vh}_t$ term, leading to:
\begin{align} \label{eq:recurrence_linear2}
\mathbf{S}_{\nxt}  &=  
    \hat{\mathbf{G}}_t \! \odot \!\mathbf{S}_{\cur} 
    - \vh_t \hat{\vk}_t^\top,
    \text{~~ where } \hat{\mathbf{G}}_t = \vone  (\mu_t   \vbeta_t + \tilde{\vh}_t \! \odot \! \hat{\vk}_t) ^\top. 
\end{align}
Here, $\hat{\mathbf{G}}_t$ is parameterized as a rank-one outer product, and hence this intra-chunk update can be computed efficiently using the parallel form of gated linear attention (GLA)~\citep{zhang2024gated}.
More details on the parallel and hardware-efficient computation is presented in Appendix \ref{app:parallel_details}.

\begin{figure}
\centering
\begin{minipage}{.49\textwidth}
  \centering
  \begin{adjustbox}{width=.9\linewidth}
      \pgfplotsset{
    compat=1.16, %
    width=14cm, %
    height=9cm, %
    every axis plot/.append style={
        line width=1.5pt, %
        solid,            %
        mark size=2.5pt,  %
    }
}

\begin{tikzpicture}
    \begin{axis}[
        title={Books Dataset},
        xlabel={Context Length},
        ylabel={Perplexity  $\downarrow$},
        xmode=log, %
        log base x=2, %
        xtick={512, 1024, 2048, 4096, 8192, 16384},
        xticklabels={512, 1k, 2k, 4k, 8k, 16k},
        grid=major, %
        ymin=16.0, %
        ymax=22.8,
        legend columns=3,
        legend pos=north west, %
        legend cell align={left},
        grid=major, 
        ymode=linear,           %
    ]

    \pgfplotstableread{
    Context Transformer LinearAtt DeltaNet GatedDelta Mamba2 TTT ModelDEC ModelENC ModelSIM
    512     20.60     21.04     20.28    19.76   19.94   20.11    19.06   19.11   19.08
    1024    19.39     20.18     19.11    18.60   18.90   19.03    17.90   18.01   17.94
    2048    18.89     19.82     18.33    18.00   18.34   18.36    17.14   17.22   17.23
    4096    18.38     19.69     17.90    17.48   18.07   18.03    16.72   16.80   16.72
    8192    18.85     20.34     18.05    17.40   18.23   18.05    16.62   16.66   16.73
    16384   nan     21.86     18.12    17.49   18.48   18.46      16.97     nan     16.82
    }\datatablebooks

    \addplot+[color=cyan, mark=x] table [x=Context, y=Transformer] {\datatablebooks}; %
    \addlegendentry{\method{Transformer}++}

    \addplot+[color=blue, mark=o, dashed] table [x=Context, y=LinearAtt] {\datatablebooks}; %
    \addlegendentry{\method{Linear-Attention}}

    \addplot+[color=green, mark=triangle, solid] table [x=Context, y=DeltaNet] {\datatablebooks};
    \addlegendentry{\method{DeltaNet}}

    \addplot+[color=orange, mark=diamond] table [x=Context, y=Mamba2] {\datatablebooks};
    \addlegendentry{\method{Mamba2}}

    \addplot+[color=black, mark=x, solid] table [x=Context, y=GatedDelta] {\datatablebooks};
    \addlegendentry{\method{Gated-DeltaNet}}

    \addplot+[color=red, mark=square, solid] table [x=Context, y=TTT] {\datatablebooks};
    \addlegendentry{\method{TTT}}

    \addplot+[color=magenta, mark=star, solid] %
        table [x=Context, y=ModelDEC] {\datatablebooks};
    \addlegendentry{\textbf{\model-DEC}}
    
    \addplot+[color=brown, mark=star, solid] %
        table [x=Context, y=ModelENC] {\datatablebooks};
    \addlegendentry{\textbf{\model-ENC}}
    
    \addplot+[color=brown!50!black, mark=star, solid] %
        table [x=Context, y=ModelSIM] {\datatablebooks};
    \addlegendentry{\textbf{\model-SIM}}

    \end{axis}
\end{tikzpicture}
  \end{adjustbox}
\end{minipage}~\hfill{}~
\begin{minipage}{.49\textwidth}
  \begin{adjustbox}{width=.9\linewidth}
    \begin{tikzpicture}
\begin{axis}[
    xlabel={FLOPs},
    ylabel={Perplexity $\downarrow$},
    xmode=log, %
    ymode=log, %
    legend pos=north east, %
    grid=major, %
    width=13cm,
    height=9cm,
    ytick={9,10,11,12,13,14,15,16,17}, 
    yticklabels={9,10,11,12,13,14,15,16,17},
    ymin=9, %
    every axis plot/.append style={
        thick, %
        mark size=2pt,
    }
]

\addplot+ [
    mark=square*,
    color=gray,
] coordinates {
    (2.548615680E+18, 16.09)
    (2.721899520E+19, 11.92)
    (1.269377741E+20, 9.75)
};
\addlegendentry{\method{Transformer}++}

\addplot+ [
    mark=x,
    color=green,
] coordinates {
    (2.577931200E+18, 15.021)
    (2.747121408E+19, 11.339)
    (1.279067400E+20, 9.47)
};
\addlegendentry{\method{\model-DEC}~(\ref{eq:decoding_recurrence})}

\addplot+ [
    mark=+,
    color= blue,
] coordinates {
    (2.574609120E+18, 15.033)
    (2.743579008E+19, 11.377)
    (1.278005400E+20, 9.47)
};
\addlegendentry{\method{\model-DOT}~(\ref{eq:similarity_recurrence})}

\addplot+ [
    mark=triangle*,
    color=red,
] coordinates {
    (2.580410880E+18, 15.5)
    (2.747335680E+19, 11.6)
    (1.279131955E+20, 9.59)
};
\addlegendentry{\method{GLA}}

\addplot+ [
    mark=o, %
    color=purple,
] coordinates {
    (2.577935520E+18, 15.378)
    (2.747124864E+19, 11.438)
    (1.279067674E+20, 9.5378)
};
\addlegendentry{\method{Gated-DeltaNet}}

\addplot+ [
    mark=diamond*,
    color=orange,
] coordinates {
    (2.574609120E+18, 16.353)
    (2.743579008E+19, 12.05)
    (1.278004608E+20, 10.03)
};
\addlegendentry{\method{DeltaNet}}

\addplot+ [
    mark=star,
    color=brown,
] coordinates {
    (2.574885840E+18, 15.596)
    (2.743803648E+19, 11.81)
    (1.278071654E+20, 9.7468)
};
\addlegendentry{\method{TTT}}

\addplot+ [
    mark=pentagon*,
    color=teal,
] coordinates {
    (2.576000160E+18, 15.675)
    (2.744691840E+19, 11.824)
    (1.278338400E+20, 9.7468) 
};
\addlegendentry{\method{Mamba2}}

\end{axis}
\end{tikzpicture}
  \end{adjustbox}
\end{minipage}
  \captionof{figure}{ \footnotesize
  Scaling patterns. \figleft: Perplexity vs. context length (110M models, Books dataset). \figright:  Perplexity vs. model size (Chinchilla-style scaling, \method{FineWeb-Edu}).
  Transformer++ results are capped at $T \le 8k$ since training them on very long contexts from scratch often performs poorly \citep{touvron2023llama}. 
  }
  \label{fig:scalling_pattern}
\end{figure}

\section{Experiments}\label{sec:expriments}
\paragraph{Setup:}
We evaluate the proposed architecture across various language modeling tasks, benchmarking its performance on both short-context and long-context tasks. 
We trained models on different datasets:  
for language modeling and common-sense reasoning tasks, models are trained on the \method{FineWeb-Edu} dataset~\citep{penedo2024fineweb} with a context length of 4k, 
for sequence length scaling analysis, models are trained on the \method{Books3} dataset (a subset of The \method{Pile} dataset~\citep{gao2020pile}) with different training context lengths ranging from 512 to 16k tokens (in increments of $ 2 \times $ per experiment).
Across all experiments, the training batch size is fixed at 0.5M tokens, irrespective of sequence length.
We compare our method against the \method{Transformer}++ baseline~\citep{touvron2023llama} as well as several state-of-the-art sub-quadratic sequence models: 
\method{Linear-Attention~(LA)}~\citep{katharopoulos2020transformers}, 
\method{TTT}~\citep{sun2024ttt},
\method{DeltaNet}~\citep{yang2024parallelizing},
\method{Gated~DeltaNet}~\citep{yang2024gatedDeltaNet},
\method{Mamba2}~\cite{mamba2} and \method{GLA}~\citep{zhang2024gated}.
Detailed experimental settings, including a training throughput comparison (\autoref{fig:throughput}), are provided in Appendix~\ref{apdx:exp-details}.

The results in \autoref{fig:scalling_pattern} (and \autoref{fig:contextlength_Pile_Books}) demonstrate that \model{} consistently achieves superior perplexity compared to all baselines across a range of context lengths.
Importantly, the performance gains of \model{} relative to other linear RNNs become more pronounced as the sequence length grows. 
This trend highlights the effectiveness of the proposed approach for long-context modeling.
Furthermore, \autoref{fig:scalling_pattern} illustrates the scaling pattern with respect to increasing model size, showing that \model achieves lower perplexity than the baselines across different parameter scales.

\paragraph{Common-sense reasoning.}
In \autoref{tab:FineWeb_340M_760_tasks}, we report the zero-shot accuracy of the trained models on various standard commonsense reasoning benchmarks—including 
PIQA~\citep{bisk2020piqa}, HellaSwag (Hella.)~\citep{zellers2019hellaswag}, WinoGrande (Wino.)~\citep{sakaguchi2021winogrande}, ARC-easy (ARC-e) and ARC-challenge (Arc-c)~\citep{clark2018think}, and Commonsense QA (CSQA)~\citep{talmor2018commonsenseqa},
commonly used for LM benchmarking~\citep{zhang2024gated, yang2024parallelizing}. As the results demonstrate, \model outperforms the baseline models on most of these tasks, achieving the highest average accuracy.

Additionally, to assess memory capacity and in-context learning capabilities, we evaluate the model on the Multi-Query Associative Recall (MQAR) task~\citep{arora2023zoology} in \autoref{apdx:exp-details}. \model demonstrates significant improvements in recall accuracy, particularly as sequence length increases, highlighting the efficiency of its compression mechanism in retaining dense information within a fixed-size recurrent memory.

\begin{table*}
\centering
\begin{minipage}[t]{.68\textwidth}
  \centering
    \centering
\footnotesize
\caption{ Performance comparison on LM and zero-shot common-sense reasoning tasks. Models are trained on \method{FineWeb-Edu} dataset.
\label{tab:FineWeb_340M_760_tasks}
}
\resizebox{0.99\linewidth}{!}{
\setlength{\tabcolsep}{3.5pt}
\renewcommand{\arraystretch}{1.3}
\begin{tabular}{ l r r r r r r r r}
\toprule
\quad\quad\textbf{Model}  & \textbf{PIQA} &    \textbf{Hella.} & \textbf{Wino.} & \textbf{ARC-e} &  \textbf{ARC-c}  & {\textbf{CSQA}} & {\textbf{BoolQ}} & \textbf{Avg.}\\
&  acc $\uparrow$  & acc\_n $\uparrow$ &   acc $\uparrow$  & acc $\uparrow$  & acc\_n $\uparrow$ &  acc $\uparrow$  &  acc $\uparrow$ &  \\
\midrule
\midrule
\multicolumn{9}{l}{\textsl{340M params~/~15B tokens}} \\
\quad \method{Transformer}++ & 66.76 & 40.40 & \underline{52.38} & 49.47 & 27.04 & 33.01 & 58.93 & 46.93  \\
\quad \method{GLA}           & 67.52 & 42.10 & 52.09 & 53.11 & 29.36 & 36.77 & 59.79 & 48.68 \\
\quad \method{Mamba2}        & 67.08 & 41.30 & 52.25 & 51.63 & 29.19 & 36.28 & 60.98 & 48.39 \\
\quad \method{DeltaNet}      & 66.21 & 40.80 & \textbf{52.64} & 50.74 & 27.73 & 36.20 & \textbf{62.05} & 48.05 \\
\quad \method{TTT}      & 66.59 & 41.10 & 51.86 & 52.14 & 26.67 & 35.71 & 60.89 & 47.88  \\
\quad \method{Gated-DeltaNet}& \underline{68.12} & 42.93 & 52.33 & {53.19} & \underline{29.44} & 36.45 & 57.74 & 48.60 \\
\arrayrulecolor{black!30}\midrule
\quad \method{\model-DEC}~(\ref{eq:decoding_recurrence}) C=1  & \textbf{68.28} & \textbf{43.33} & 51.70 & 53.66 & 28.41 & \underline{37.27} & 60.70 & \underline{49.05} \\
\quad \method{\model-DEC}~(\ref{eq:decoding_recurrence}) C=4  & 67.25 & \underline{43.20} & 51.93 & \textbf{54.33} & \textbf{29.70} & 36.12 & \underline{61.44} & \textbf{49.14} \\
\quad \method{\model-DOT}~(\ref{eq:similarity_recurrence}) C=4  & 67.95 & 43.13 & 51.07 & \underline{53.91} & 28.50 & \textbf{38.33} & 57.71 & 48.66 \\
\arrayrulecolor{black}\midrule
\multicolumn{9}{l}{\textsl{760M params~/~30B tokens}} \\
\quad \method{Transformer}++ & 69.59 & 51.13 & 53.83 & 56.96 & 30.30 & 39.64 & \textbf{62.11} & 51.94 \\
\quad \method{GLA} & 69.64 & \underline{52.17} & 53.20 & 59.07 & 32.62 & 41.20 & 60.12 & 52.57 \\
\quad \method{Mamba2}      & 70.18 & 50.67 & 52.64 & 59.28 & 33.99 & 41.77 & 57.19 & 52.25 \\
\quad \method{DeltaNet}      & 69.86 & 48.60 & 51.85 & 58.10 & 32.36 & 40.05 &  58.56 & 51.34 \\
\quad \method{TTT} & 70.08 & 50.67 & 52.09 & 58.65 & 33.73 & 41.28 & 60.43  & 52.42 \\
\quad \method{Gated-DeltaNet} & \underline{71.38} & 51.87 & \underline{53.99} & \textbf{61.06} & 34.08 & 39.97 & 56.18 & 52.64 \\
\arrayrulecolor{black!30}\midrule
\quad \method{\model-DEC}~(\ref{eq:decoding_recurrence}) & 71.00 & \textbf{52.93} & 53.28 & \underline{59.66} & \underline{34.68} & \textbf{43.90} & 59.33 & \underline{53.54}\\
\quad \method{\model-DOT}~(\ref{eq:similarity_recurrence}) & \textbf{71.76} & 51.70 & \textbf{54.38} & 59.49 & \textbf{36.57} & \underline{43.57} & \underline{60.73} & \textbf{54.03} \\
\arrayrulecolor{black}\bottomrule
\end{tabular}
}

\end{minipage}~\hfill{}~
\begin{minipage}[t]{.31\textwidth}
  \centering
    \footnotesize
\caption{ \label{tab:ablation}
Ablation study of \model{}'s components (110M parameters, trained on \method{FineWeb-Edu}). Each row starting with \textit{+} adds a new component to the configuration in the row above it, beginning from the DeltaNet baseline. The final row with \textit{+} represents the full \model{} configuration.
Unless otherwise stated, the memory size is set to $m\!=\!d\!=\!64$ in all models. 
The table also includes comparisons of different parallel approximations ($C\!=\!4$) and evaluates the impact of memory size $m$ on the compression layer ($C\!=\!1$). 
}
\vskip -5pt
\begin{adjustbox}{width=.98\linewidth}
\setlength{\tabcolsep}{4pt}
\renewcommand{\arraystretch}{1.3}
\centering
\begin{tabular}{lc}
        \toprule
        \method{\textbf{Configuration}} & ppl $\downarrow$ \\
        \midrule
        \method{DeltaNet} & 16.35  \\    
        \method{TTT~C\!=\!1} & 15.60  \\        
        \arrayrulecolor{black!30}\midrule
        \method{\model~C\!=\!1} \\
        \quad + \method{orthogonal~recurrence} (\ref{eq:decoding_recurrence})  &   15.38  \\
        \quad + \method{normalized~projection} (\ref{eqn:normalized_recurrence}) &  15.07  \\
        \quad + \method{forget-gate} (\ref{eq:dacay_recurrence}) &  15.02  \\
        \arrayrulecolor{black!30}\midrule
        \method{\model~C\!=\!4} approx.~\eq{eq:recurrence_linear1}& 15.15 \\
        \method{\model~C\!=\!4} approx.~\eq{eq:recurrence_linear2}&  15.12 \\
        \arrayrulecolor{black!30}\midrule
        \method{{\model~m\!=\!16}}&  15.52 \\
        \method{{\model~m\!=\!32}}&  15.26 \\
        \method{{\model~m\!=\!128}}&  14.84 \\
        \method{{\model~m\!=\!192}}&  14.71 \\

    \arrayrulecolor{black}\bottomrule
\end{tabular}
\end{adjustbox}

\end{minipage}
\end{table*}

\paragraph{Ablation.}
We ablate key components of the \model to evaluate the contribution of each to the overall performance. 
The results in \autoref{tab:ablation} underscore the significance of orthogonal state recurrence (\ref{eq:similarity_recurrence}) and normalized projection (\ref{eqn:normalized_recurrence}) to the model's overall performance. 
Furthermore, our analysis indicates that the explicit forget gate's impact on the overall performance is negligible, suggesting that the normalized projection introduced in \eq{eqn:normalized_recurrence} inherently acts as an effective forgetting mechanism. 
Finally, \model with only a quarter of the memory slots ($m=16$) outperforms \method{TTT} with $m=d=64$, validating that the proposed orthogonal update mechanism utilizes the memory capacity significantly more efficiently.
Extended ablation studies on scaling memory size and per-token vs. per-chunk normalization are presented in \autoref{apdx:exp-details}.

\section{Conclusion}
This work introduced a novel recurrent neural network mechanism designed for efficient information compression into a matrix-valued state with a limited number of memory slots. 
We approached this problem by framing it as an online optimization problem, deriving the memory's dynamic update rule from a single gradient descent step.
The resulting recurrence features a state- and input-dependent gating mechanism, leading to an interpretable memory update process.  
A core feature of this mechanism is that each memory slot is updated exclusively with information that is orthogonal to its current state. This orthogonal update ensures that only new, non-redundant data is written into memory and minimizes the interference with previously stored information.
Furthermore, the update includes an input- and state-dependent writing intensity, providing fine-grained control over the magnitude of the information written to each memory slot.
With its sub-quadratic complexity, this mechanism offers a promising alternative to Transformers for pre-training or a method for efficiently fine-tuning pre-trained Transformers into RNNs.

\bibliography{references}
\bibliographystyle{plainnat}%

\clearpage
\appendix
\onecolumn

\section{Discussion and Related Works} \label{apdx:related}

\paragraph{Fast Weight Programmers and Test-Time Training.}
The two-stage learning process adopted in our work draws inspiration from the concept of Fast Weight Programmers (FWPs)~\citep{schmidhuber1992learning, schlag2021linear} where a "slow" network dynamically updates the parameters of a "fast" network. 
In our framework, the compression layer in the inner loop can be seen as the fast network, with its memory states,  $\mathbf{S}_t$, acting as "fast weights" that are rapidly adapted to the evolving contextual information. 
The outer loop, conversely, learns the generalizable parameters of the slow neural network, optimized across the entire training dataset.
The continual reprogramming of fast network weights by slow models~\citep{irie2021going,clark2022meta} is broadly recognized as Fast Weight Programming, also referred to as synaptic modulation~\citep{von1994correlation} or input-dependent parameterization~\citep{karami2019invertible, gu2023mamba,karami2024orchid}, a technique known to enhance model expressiveness. 
In our architecture, the parameterization of the linear projections by the slow network facilitates this fast adaptation within the inner loop.
Similarly, \textit{Test-Time Training}~\citep{sun2020test,sun2024ttt,titans2024,von2025mesanet} is a paradigm where a model adapts to each test instance by optimizing a self-supervised objective before making predictions. Our compression layer effectively implements a form of test-time training by dynamically updating its state based on the contextual information of the input sequence during inference. 
In contrast to the aforementioned works, our approach introduces an explicit learning mechanism for the "fast" compression layer, leading to an interpretable update rule for its internal states that optimally compresses the latest token into memory at test time.

Test-time training for sequence models have recently been formalized through the lens of online optimization.
The works such as  
~\cite{sun2024ttt, liu2024longhorn, titans2024, behrouz2025atlas, wang2025testTimeRegression, karami2025trellis, von2025mesanet, zhang2025TTTDR} fall under this category. They have demonstrated that deriving recurrent update rules from the online optimization of a regression objective can yield powerful sequence models.

\paragraph{Adaptive Filters.}
Classical adaptive filtering algorithms~\citep{haykin2002adaptive} iteratively update their weights to minimize prediction error while efficiently adapting to streaming, non-stationary data. These methods share core principles with the online learning and dynamic memory updates employed in our work. In particular, the gradient descent-based update rules we adopted for memory adaptation are closely related to the Least Mean Squares (LMS) algorithm—also known as the Widrow-Hoff algorithm~\citep{widrow1988adaptive}—which updates weights using the instantaneous gradient of the squared error.  
Furthermore, variations such as Normalized Least Mean Squares (NLMS), which involves a normalized step size for improved convergence, and Leaky LMS, which incorporates a leakage factor used to prevent unbounded growth of filter weights, find parallels in our use of normalization mechanisms to stabilize memory update (\autoref{eqn:normalized_recurrence}) and state decay (\autoref{eq:dacay_recurrence}). 
While these adaptive filtering methods rely on linear weight updates, our approach introduces a non-linear memory update rule that incorporates only the non-redundant components of the new token.

\paragraph{Matrix Factorization}
Matrix factorization and dictionary learning are classical representation learning techniques that aim to extract essential features from complex data by approximating it as a linear combination of a reduced set of basis vectors, also known as dictionary atoms. This concept is also conceptually related to topic modeling, where the objective is to extract important features (topics) from a complex dataset to obtain a reduced representation~\citep{blei2009topic, blei2012probabilistic}.
\citet{mairal2010online} proposed an online optimization algorithm for structured matrix factorization and sparse coding for i.i.d. stream of data, which efficiently scales to large datasets.
Subsequently, \citet{lyu2020online} extended this work by proving the convergence of such an online algorithm in non-i.i.d. settings, where the sequential data forms a Markov chain.
In a related area, \citet{karami2017multi} formulated the identification of SSMs (a.k.a. linear dynamical systems) as a multi-view matrix factorization problem and proposed a convex optimizer for its solution. 
In contrast to  the online matrix factorization in~\cite{mairal2009online}, which employs a model-free method to learn the latent coefficients (codes) and leverages block coordinate descent for optimization, our method formulates the memory update as a fast internal optimization procedure. We incorporate a simple encoding layer to generate the latent representation, $\vrep_t$, and integrate it into a larger deep neural network training procedure.

\paragraph{Recent Advancements in RNNs.}
More broadly, the field is witnessing a resurgence of interest in RNN research, with many new architectures emerging as viable alternatives to Transformers. Works such as~\cite{sun2023retentive, lin2025forgettingTransformer, merrill2024illusion, kacham2024polysketchformer, peng2023rwkv, peng2025rwkv,  guo2025log, du2025mom, siems2025deltaproduct, hu2025comba, karami2025MSSSM, behrouz2025atlas,zhang2025TTTDR} tackle the limitations of RNNs, introducing new mechanisms to improve long-range dependency, handling scalability, or efficient parallelizability.

\begin{remark2}[\textbf{Delta Rule}]
Removing the non-linearity $\phi(\cdot)$ from the compression layer simplifies the online gradient descent update rule in (\ref{eqn:OGD}) to
\begin{align} \label{eqn:DeltaNet}
    \mathbf{S}_{\nxt} &= \mathbf{S}_{\cur} -  \gamma_t
    (\mathbf{S}_{\cur} {\vrep}_{t} - {\vv_t}) \vrep_t^\top  
    = \mathbf{S}_{\cur} (\mathbf{I} - \gamma_t {\vrep}_{t} \vrep_t^\top ) +  \gamma_t {\vv_t} \vrep_t^\top 
\end{align}
This linear update rule recovers the delta rule~\citep{widrow1988adaptive}, known for its higher memory capacity~\citep{prados1989neural} and has been demonstrated as an effective form of linear recurrence, particularly in associative recall tasks~\citep{schlag2021linear, yang2024parallelizing}.
Similar to linear transformers, the second term writes into memory via the outer product ${\vv_t} \vrep_t^\top$, while the first term implements a forgetting mechanism, controlled by the new key $\vrep_t$, to remove old information from memory.
Here, we propose a more efficient update rule based on the nonlinear interactions between the memory and the non-redundant information of the new keys.
\end{remark2}

\begin{remark2}[Layer Normalization]
In the proposed compression layers,  normalization was applied to each state column.
Alternatively, \citet{sun2024ttt} proposed applying normalization on the output of the decoding layer.\footnote{While the general formulation of \method{TTT}~\citep{sun2024ttt} applies a non-linearity to $\vz_t$, their implementation specifically utilizes normalization.}
In this case, the decoding function becomes: 
$\hat{\vv_t} = g({\vrep}_t; \mathbf{S}_t) = \phi(\mathbf{S}_t {\vrep}_t)$.
This formulation is analogous to applying \emph{layer normalization} as commonly used in deep neural networks.
As before, we can simplify the gradient $\nabla_{\mathbf{S}} \mathcal{L}_t$ to derive an interpretable update rule. Specifically, let $\phi(\vz_t)= \frac{\vz_t}{\|\vz_t \|}$, where $\vz_t :=  \mathbf{S}_{\cur} {\vrep}_{t}$, and define the reconstruction error as $\ve_t := \hat{\vv_t} - {\vv_t}$. Applying the chain rule, we obtain:
\begin{align}
\frac{\partial \mathcal{L}_t}{\partial \mathbf{S}} = \ve_t^\top \, \mathbf{J}_\phi (\vz_t)
\vrep_t^\top, 
\text{ where } 
\mathbf{J}_\phi  (\vz_t) =  \frac{\mathbf{P}({\vz}_t)}{\|\vz_t\|} = \frac{1}{\|\vz_t\|} \left( 
\mathbf{I} -
\frac{\vz_t {\vz}_t^\top}{\|\vz_t\|^2}  \right) \nonumber
\end{align}
Subsequently, the gradient descent update follows a nonlinear recurrence:
\begin{align}
\mathbf{S}_{\nxt} 
= \mathbf{S}_{\cur} - \gamma_t \ve_t^\top \, \frac{\mathbf{P}({\vz}_t)}{\|{\vz}_t\|}
\vrep_t^\top
\end{align}
This nonlinear state recurrence incorporates an outer-product correction based on the projection of the reconstruction error $\ve_t$ onto the orthogonal complement space of $ \hat{\vv}_t$.

The concept of normalizing the state vectors in our compression model, as described in \S \ref{sec:dec}, shares similarities with weight normalization techniques used in deep learning literature~\citep{salimans2016weight}.
Furthermore, the two interpretations presented above—applying normalization to the output of the decoding layer versus normalizing the state vectors—offer insights into the rationale behind different normalization schemes commonly used in deep learning, such as weight normalization and layer normalization~\citep{ba2016layer}.
Each normalization method plays a distinct role in stabilizing training and improving generalization.
\end{remark2}

\todoM[]{Clarify}

\section{Detailed Derivations and Model Variants} \label{apdx:details}

\subsection{Parallel and Hardware Efficient Form} \label{app:parallel_details}
Various methods have been explored to enable parallel evaluation of non-linear RNNs. One strategy, as proposed by \citet{lim2023parallelizing,gonzalez2024towards}, involves casting inference as finding the solution to a fixed-point equation, thereby achieving parallelism.
 In a different approach, \citet{sun2024ttt} introduced a parallel chunk-wise solution using a gradient approximation. This method splits a sequence into non-overlapping chunks and utilizes the state at the beginning of each chunk to approximate the gradients for the entire chunk in parallel. 
Following this technique, we let  $\mathbf{S}_{t'}$ represent the state at the beginning of the chunk (i.e., the final state from the preceding chunk), where $t' \!=\! t - \op{mod}(t, C)$ with $C$ denoting the chunk size. The gradient is then approximated as 
$\nabla_{S} \mathcal{L}(\mathbf{S}_{t'}, \vv_t, \vrep_t)$.This approximation linearizes the nonlinear recurrence \eq{eq:recurrence_detailed} as:
\begin{align} \label{eq:recurrence_linear1_apdx}
\mathbf{S}_{\nxt}  &=  
    \mathbf{G}_t \! \odot \!\mathbf{S}_{\cur} 
    + \vone (\tilde{\vh}_t \! \odot \! \hat{\vk}_t)^\top \! \odot \!\mathbf{S}_{t'}
    - \vh_t \hat{\vk}_t^\top,
\end{align}
where $\tilde{\vh}_t \!=\! \mathbf{S}_{t'}^\top \vh_t $.
Now, for the time steps $t\!=\! b C \!+\! \tau$ in the $b$-th block, let $\mathbf{X}^b = \vx[bC \!+\! 1 \!:\! b(C \!+\! 1)]$ denote the stacked input into the chunk-wise matrices and $ \mathbf{X}_{\tau}^b \!=\! \vx[bC+\tau]$
(similarly for other vectors such as $\vq, \vh, \vk$),
and define the local cumulative product of decay factors  as 
$\va_\tau^b = \prod_{i = b C+1}^{b C+ \tau} \vbeta_i \mu_i$ .
We also define a block lower triangular tensor $\MOmega^b \in \mathbb{R}^{C \times C \times m}$ with components $\MOmega^b_{j,i,:} = \frac{\va_j^b}{\va_i^b}  \mathbb{I}_{i \le j}$ that are segmented cumulative product over the sub-block steps $i$ to $j$ ($1\le i\le j\le C$). Here $\mathbb{I}_{i \le j}$ is the indicator function (equal to 1 if $i \le j$ and 0 otherwise), the division $\frac{\va_j^b}{\va_i^b}$ is performed element-wise.
Therefore, the recurrence \eqref{eq:recurrence_linear1} can be expressed at the chunk-level as:
\begin{align} \label{eq:chunkwise_OSR_apdx}
\mathbf{S}_{b}  &=  
    \big(\vone ({\va_C^b} \!+\! {\vf^b})^\top \big)\! \odot \! \mathbf{S}_{b-1}
    - {\mH^b}^\top (\hat{\mK}^b \! \odot \! \MOmega^b_{C,:,:})) \\
    \mY^b &= \big(\mQ^b \! \odot \! (\mLambda^b \!+\! \mF^b) \big) \mathbf{S}_{b-1}^\top
    - \mP^b \mH^b  \nonumber
\end{align}
where $\vf^b \!=\! \op{diag}[\tilde{\mH}^b \, (\hat{\mK}^b \! \odot \! \MOmega^b_{C::})]  \in \mathbb{R}^{m}$ 
, and $\mF^b \in \mathbb{R}^{C \times m}$ is a matrix with entries
$\mF^b_{i n} \!=\! \sum_{\tau=1}^{C} (\tilde{\mH}^b \! \odot \! \hat{\mK}^b)_{\tau n} \, \MOmega^b_{i \tau n} ~\forall 0<i \le C,~ 0<j \le m$.
Both of these can be computed using \textit{tensor contraction} or Einstein summation operation as
\begin{align*}
\vf^b &\!=\! \sum_{\tau=1}^{C} \tilde{\mH}^b_{\tau:} \, (\hat{\mK}^b \! \odot \! \MOmega^b_{C::})_{\tau:} \!=\! \texttt{einsum}(\texttt{"C m, C m -> m"},~ \tilde{\mH}^b,~ (\hat{\mK}^b \! \odot \! \MOmega^b_{C::})) \\
\mF^b &\!=\! \texttt{einsum}(\texttt{"C Ci m, Ci m -> C m"},~ \! \MOmega^b,~ (\tilde{\mH}^b \! \odot \! \hat{\mK}^b))).
\end{align*}
Here, $\mP^b \in \mathbb{R}^{C \times C}$ is a lower triangular matrix with $\mP^b_{ij} = \sum_{k=1}^{m} \mQ^b_{ik} \, \hat{\mK}^b_{jk} \, \MOmega^b_{ijk}$. 
A \texttt{matmul}-optimal computation for $\mP^b$ is presented in \cite{zhang2024gated} using a sub-tiling technique.   
However, a closer inspection of \eq{eq:recurrence_detailed} reveals that the recurrence can be simplified by linearizing only the $\hat{\vh}_t$ term, leading to:
\begin{align} \label{eq:recurrence_linear2_apdx}
\mathbf{S}_{\nxt}  &=  
    \hat{\mathbf{G}}_t \! \odot \!\mathbf{S}_{\cur} 
    - \vh_t \hat{\vk}_t^\top,
    \text{~~ where } \hat{\mathbf{G}}_t = \vone  (\mu_t   \vbeta_t + \tilde{\vh}_t \! \odot \! \hat{\vk}_t) ^\top. 
\end{align}
Here, $\hat{\mathbf{G}}_t$ is parameterized as a rank-one outer product, and hence this intra-chunk update can be computed efficiently using the parallel form of gated linear attention (GLA) presented in \cite{zhang2024gated}.

\paragraph{Computing $\vbeta$.} 
For the normalization factor, $\vbeta_{\tau}$, in \eqref{eqn:normalized_recurrence}, we need to compute $\|\Delta \mathbf{s}_{i,\tau}\|^2 ~~ \forall~ 0< i \le m, ~ 0< \tau \le C$ which require materializing $\Delta \mS_{\tau} \in \mathbb{R}^{d \times m}~\forall ~ 0< \tau \le C $, that is memory and compute inefficient. However, we can see that these norms can be efficiently computed using matrix multiplication:
\begin{align*}
    \Delta \mS_{\tau} &= -\gamma_t \nabla_{S} \mathcal{L}(\mathbf{S}_{\cur}, \vv_t, \vrep_t) = \vone (\tilde{\vh}_t \! \odot \! \hat{\vk}_t)^\top \! \odot \!\mathbf{S}_{0}
    - \vh_t \hat{\vk}_t^\top \\
    \vbeta_{\tau} &= \left(1+ \op{diag}[\Delta \mS_{\tau}^\top  \, \Delta \mS_{\tau} ] \right)^{-\frac{1}{2}} \\
    \vd_{S_\tau} :=\op{diag}[\Delta \mS_{\tau}^\top  \, \Delta \mS_{\tau} ] & =
    \op{diag}[\mS_{0}^\top  \, \mS_{0} ] \! \odot \!  (\tilde{\vh}_t \! \odot \! \hat{\vk}_t)^2 + \|\vh_t \|^2  \hat{\vk}_t^2 
\end{align*}
We can also compute $\vd_{S_\tau} ~\forall ~ 0< \tau \le C$ for the entire chunk by matrix multiplications as 
\begin{align} \label{eq:beta_chunk}
    \mD_S := [\vd_{S_0}, \dots, \vd_{S_C}] &= \op{diag}[\mS_{0}^\top  \, \mS_{0} ] \! \odot \!  (\tilde{\mH} \! \odot \! \hat{\mK})^2 + \op{diag}[\mH \mH^\top ]  \hat{\mK}^2 ~ \in \mathbb{R}^{C \times m} \nonumber \\
    [\vbeta_0, \dots, \vbeta_C] &= \left(1+ \mD_S \right)^{-\frac{1}{2}}
\end{align}
where all the square  and square roots are performed element-wise.

This matrix form computes the states only at the end of each block massively save computation and I/O overhead, while enabling efficient use of the \texttt{matmul} operations on modern accelerators~\citep{hua2022transformer, kacham2024polysketchformer, mamba2, zhang2024gated, sun2023retentive}.

\subsection{Encoding layer} \label{sec:encoding}
Principal Component Analysis (PCA) can be formulated as a linear regression problem, where the data is projected onto a lower-dimensional latent space
~\citep[\S 5.8]{goodfellow2016deep}
Inspired by this regression perspective, we define a encoding layer as:
$\hat{\vrep}_t = f({\vv}_t; \mathbf{S}_t) = \phi(\mathbf{S}_t)^\top\, \vv_t $
with the corresponding $\ell_2$ loss:
\begin{align}
\mathcal{L}_t = \left\| \phi(\mathbf{S}_t)^\top\, \vv_t - \vk_t \right\|^2, ~~   \mathbf{S}_t \in \mathbb{R}^{d \times m}, ~ \vv_t \in \mathbb{R}^{d}, ~ \vrep_t \in \mathbb{R}^{m}
\label{eq:encoding_stateNorm_loss}
\end{align}
From this, we can derive a closed-form expression for the gradient, resulting in the recurrence:
\begin{align}
\mathbf{S}_{\nxt} = \mathbf{S}_{\cur} - \gamma_t \vv_t^\top
\times_1
\begin{bmatrix} 
  \frac{\mathbf{P}({\vs}_1)}{\|\mathbf{s}_1\|}, \dots,     \frac{\mathbf{P}({\vs}_m)}{\|\mathbf{s}_m\|} 
\end{bmatrix} \odot \ve_t^\top
\label{eq:encoding_recurrence}
\end{align}

\paragraph{General form.}
The formulations presented in Eqs.~(\ref{eq:decoding_recurrence}, \ref{eq:encoding_recurrence}, and \ref{eq:similarity_recurrence}) offer principled approaches for designing compression layers in our framework. 
In general, we refer to this update rule as \emph{Orthogonal State Recurrence (OSR)},
which unifies compression layer into a common framework formulated as follows
\begin{align} \label{eq:SDC_general_stateNorm}
\{\vy_t\}_{t=1}^T &= 
\op{\model}
(\{\vk_t, \vv_t, \vq_t\}_{t=1}^T)   
:= 
\begin{cases}
    \mathbf{S}_{\nxt} = \mathbf{S}_{\cur} - \gamma_t \vh_t^\top
    \times_1
    \begin{bmatrix} 
        \frac{\mathbf{P}({\vs}_1)}{\|\vs_1\|}, \dots,     \frac{\mathbf{P}({\vs}_m)}{\|\vs_m\|} 
    \end{bmatrix} \odot \vc_t^\top  
    \\
    \vy_t = \mathbf{S}_{\nxt} {\vq}_{t}   
\end{cases} 
\end{align}
Here, the definitions of  $\vh_t$ and $\vc_t$ vary depending on the specific layer:
$$\begin{cases}
\{\vh_t = \ve_t, \quad \vc_t = \vrep_t\}, & \text{Decoding Layer} \\
\{\vh_t = \vv_t, \quad \vc_t = \ve_t\}, & \text{Encoding Layer} \\
\{\vh_t = -\vv_t, \quad \vc_t = \vrep_t\}, & \text{Similarity Objective}
\end{cases}$$
Table \ref{tbl:ogd-list}  provides a summary comparing the online gradient descent-based recurrent corresponding to the proposed compression layers and those of existing RNNs.

\subsection{Proofs} \label{apdx:proofs}

\subsubsection{Encoder Layer with State Normalization}

The reconstruction loss in this case is

\[
\mathcal{L}_t = \left\| \phi(\mathbf{S}_t)^\top\, \vv_t - \vk_t \right\|^2,
\]
where
$\mathbf{S}_t \in \mathbb{R}^{d \times m}$ with columns $\vs_i$ (for $i=1,\dots, m$), $\vv_t \in \mathbb{R}^d, ~\vk_t \in \mathbb{R}^{m}$
and $\phi(\mathbf{S}_t) = [\vphi_1, \ldots, \vphi_m]$ is obtained by normalizing each column of $\mathbf{S}_t$; that is, 
$ \vphi_i = \frac{\vs_i}{\|\vs_i\|}$.
Decomposing the reconstruction error in a per-basis (per-column) form and defining the reconstruction error,
\[
\ve_i := \vphi_i^\top\, \vv_t - (\vk_t)_i ~ \forall~ i=1,\dots, m
\]
Then the loss is $\mathcal{L}_t = \sum_{i=1}^m \ve_i^2$.
By this decomposition, we can derive the gradient with respect to each column $\vs_i$ separately:
\[
\frac{\partial \mathcal{L}_t}{\partial \vs_i} = 2\, \ve_i \, \frac{\partial \ve_i}{\partial \vs_i}.
\]
The Jacobian of the normalized vector
$
\vphi_i = \frac{\vs_i}{\|\vs_i\|},
$
is
\[
\nabla_{\vs_i} \vphi_i = \frac{1}{\|\vs_i\|} \left( \mathbf{I}_d - \vphi_i \vphi_i^\top \right).
\]
Thus, by the chain rule,
\[
\frac{\partial (\vphi_i^\top\, \vv_t)}{\partial \vs_i} =  \frac{\partial \vphi_i}{\partial \vs_i}\, \vv_t = \frac{1}{\|\vs_i\|}  \left( \mathbf{I}_d - \vphi_i \vphi_i^\top \right) \, \vv_t.
\]
Therefore, for each $i$ the gradient with respect to $\vs_i$ is
\[
\frac{\partial \mathcal{L}_t}{\partial \vs_i} 
= \frac{2\, \ve_i }{\|\vs_i\|} \mathbf{P}({\vs}_i) \, \vv_t
= 2\, \ve_i\, \frac{1}{\|\vs_i\|} \left( \vv_t -  \frac{\vs_i ~ (\vs_i^\top \vv_t)}{\|\vs_i\|^2}  \right),\quad \text{for } i=1,\dots, m. 
\]
Here, the matrix
$\mathbf{P}({\vs}_i) = \mathbf{P}({\vphi}_i) := \left( 
\mathbf{I} -
\frac{\vs_t {\vs}_i^\top}{\|\vs_i\|^2}  \right)  $ 
is known as the \textit{projection matrix onto the orthogonal complement} of \( {\vs}_i\) in linear algebra~\citep[\S 3.3]{strang2000linearAlgebra}. 
Stacking these column gradients into the gradient with respect to the matrix $\mathbf{S}$
we obtain
\[
\nabla_{\mathbf{S}} \mathcal{L}_t = \left[ \frac{\partial \mathcal{L}_t}{\partial \vs_1},\, \frac{\partial \mathcal{L}_t}{\partial \vs_2},\, \dots,\, \frac{\partial \mathcal{L}_t}{\partial \vs_m} \right],
\]

Thus, the closed-form gradient can be expressed as:
\begin{align}
\nabla_{\mathbf{S}} \mathcal{L}_t &= \left[ \frac{2 \Bigl( \vphi_1^\top\, \vv_t - (\vk_t)_1 \Bigr)}{\|\vs_1\|} \left( \mathbf{I}_d - \vphi_1 \vphi_1^\top \right) \vv_t,\;\cdots,\; \frac{2 \Bigl( \vphi_m^\top\, \vv_t - (\vk_t)_m \Bigr)}{\|\vs_m\|} \left( \mathbf{I}_d - \vphi_m \vphi_m^\top \right) \vv_t \right] \\
&= \left[ \frac{2\, \ve_1 }{\|\vs_1\|} \mathbf{P}({\vs}_1) \, \vv_t,\;\cdots,\; \frac{2\, \ve_m }{\|\vs_m\|} \mathbf{P}({\vs}_m) \, \vv_t \right] 
\tag*{\hfill$\square$}
\end{align}

\begin{proof}
[Proof of Proposition~\ref{thm:Riemannian}]

Consider the unit sphere  
$\mathcal{C} = \{ \vs \in \R^d \mid \|\vs\| = 1 \}$ which is a which is a smooth Riemannian manifold. 
Let 
 $\nabla_{s} \ell(\mathbf{s})$, the gradient of the loss $\ell$, and let $\nabla_{\mathcal{C}} \ell(\mathbf{s})$ be its orthogonal projection from the ambient space $\R^d$ onto the tangent space of the manifold at $\vs$, denoted by $\mathcal{T}_{\mathcal{C}}(\vs)$.

In our update rule, the gradient term
$
\Delta \vs = \alpha\, \mathbf{h}^{\perp\, \vs}
$
is constructed such that it is orthogonal to $\vs$; that is, $\Delta \vs \in \mathcal{T}_{\mathcal{C}}(\vs) $. Hence, we have
\[
\nabla_{\mathcal{C}} \ell(\mathbf{s}) = \nabla_{s} \ell(\mathbf{s})
\]

The gradient descent update of $\ell$ on the Riemannian  manifold $\mathcal{C}$ is given by
\[
\mathbf{s}_{\text{new}} = \exp_{\mathbf{s}}\Bigl(-\eta_t\, \nabla_{\mathcal{C}} \ell(\mathbf{s})\Bigr),
\]
where \(\exp_{\mathbf{s}}\) is the exponential map on the sphere $\mathcal{C}$~\citep{absil2009optimization, boumal2023introduction}.

Replacing the exponential map with its first-order approximation, called retraction step, which projects from the tangent space onto the sphere manifold~\citep{bonnabel2013stochastic}. 
Therefore, our projected gradient update of the form $\vs_{i,\nxt} = \mathcal{P}_{\mathcal{C}}(\vs_{i,\cur} + \Delta \vs_{i,t})$ (\eqref{eqn:normalized_recurrence})
is equivalent to performing a Riemannian gradient descent step with retraction on the manifold $\mathcal{C}$.   
 $\square$
\end{proof}

\begin{figure}[t]
\centering
\begin{adjustbox}{width=.45\linewidth}
    \begin{tikzpicture}[scale=1.2, every node/.style={font=\small}]

  \draw[->] (-.5,0) -- (3,0) node[right] {$x$};
  \draw[->] (0,-1.) -- (0,3) node[above] {$y$};

  \coordinate (O) at (0,0);
  \coordinate (s) at (1,2);
  \coordinate (v) at (1,.5);
  \coordinate (vproj) at (.6, -0.3);
  
  \coordinate (snext) at (1.6, 1.7);
  
  \coordinate (snext0) at (3., 3.);
  
  \draw[->,  very thick, red] (O) -- (s) node[midway, above left] {$\vs_{\cur}$};
  
  \draw[->, very thick, blue] (O) -- (v) node[below ] {$\vh_t$};
  
  \draw[->, very thick, green!70!black] (O) -- (vproj) node[below] {$\vh_t^{\perp {\vs}_{\cur}} = \mathbf{P}({\vs}_{\cur}) \, \vh_t$};
  
  \draw[dashed] (v) -- (vproj);
  
  \draw[dashed, orange, domain=-.5:3] plot (\x, {-0.5*\x}) node[below] {
  \tiny Orth. complement 
  space of $\vs_{\cur}$
  };

\node at (1,-1.8) {\captiona};

\begin{scope}[xshift=4cm]

  \draw[->] (-.5,0) -- (3,0) node[right] {$x$};
  \draw[->] (0,-.5) -- (0,3) node[above] {$y$};

  \coordinate (O) at (0,0);
  \coordinate (s) at (1,2);
  \coordinate (v) at (1,.5);
  \coordinate (vproj) at (.6, -0.3);
  
  \coordinate (snext) at (1.6, 1.7);
  
  \coordinate (snext0) at (2., 2.5);
  
  \draw[->,  very thick, red] (O) -- (s) node[midway, above left] {$\vs_{\cur}$};

  \draw[->, very thick, green!70!black] (O) -- (vproj) node[below] {$\vh_t^{\perp {\vs}_{\cur}} = \mathbf{P}({\vs}_{\cur}) \, \vh_t$};
  
  \draw[dashed] (snext) -- (vproj);

  \draw[dashed] (snext) -- (s);

  \draw[->, very thick, purple] (O) -- (snext) node[above] {$\vs_{\nxt}$};
    
  \draw[->, very thick, dashed, violet] (O) -- (snext0) node[above] {$\hat{\vs}_{\nxt}$};
\node at (1.5,-1.8) 
{\captionb};

\end{scope}

\end{tikzpicture}
\end{adjustbox}
\begin{adjustbox}{width=.45\linewidth}
    \begin{tikzpicture}[scale=1.2, every node/.style={font=\small}]

  \draw[->] (-.5,0) -- (3,0) node[right] {$x$};
  \draw[->] (0,-1.5) -- (0,2) node[above] {$y$};

  \coordinate (O) at (0,0);
  \coordinate (s) at (1,2);
  \coordinate (v) at ( 1.5, -2.);
  \coordinate (vproj) at (2., -1.);
  
  \coordinate (snext) at (3., 1.);
  
  \coordinate (snext0) at (2.5, 0.);
  
  \draw[->,  very thick, red] (O) -- (s) node[midway, above left] {$\vs_{\cur}$};
  
  \draw[->, very thick, blue] (O) -- (v) node[below ] {$\vh_t$};
  
  \draw[->, very thick, green!70!black] (O) -- (vproj) node[above right] {$\vh_t^{\perp {\vs}_{\cur}} = \mathbf{P}\vh_t$};
  
  \draw[dashed] (v) -- (vproj);
  
  \draw[dashed, orange, domain=-.5:3] plot (\x, {-0.5*\x}) node[below] {
    \tiny Orth. complement 
  space of $\vs_{\cur}$
  };

\node at (1,-2.2){\captiona};

\begin{scope}[xshift=4cm]

  \draw[->] (-.5,0) -- (3,0) node[right] {$x$};
  \draw[->] (0,-1.5) -- (0,2) node[above] {$y$};

  \coordinate (O) at (0,0);
  \coordinate (s) at (1,2);
  \coordinate (v) at ( 1.5, -2.);
  \coordinate (vproj) at (2., -1.);
  
  \coordinate (snext) at (3., 1.);
  
  \coordinate (snext0) at (2.5, 0.);
  
  \draw[->,  very thick, red] (O) -- (s) node[midway, above left] {$\vs_{\cur}$};

  \draw[->, very thick, green!70!black] (O) -- (vproj) node[below] {$\vh_t^{\perp {\vs}_{\cur}} = \mathbf{P}\vh_t$};
  
  \draw[dashed] (snext) -- (vproj);

  \draw[dashed] (snext) -- (s);

  \draw[->, very thick, purple] (O) -- (snext) node[above] {$\vs_{\nxt}$};
    
  \draw[->, very thick, dashed, violet] (O) -- (snext0) node[above] {$\hat{\vs}_{\nxt}$};
\node at (1.5,-2.2) 
{\captionb};

\end{scope}

\end{tikzpicture}
\end{adjustbox}
\caption{  \label{fig:ortho_schemas2}
An illustration of the proposed update rule. 
(a) Example of a single memory slot state, $\vs_t$, an incoming token representation, $\vh_t$, and its component orthogonal to the current state, $\vh_t^{\perp {\vs}_{\cur}}$.
(b) The updated state according to the proposed update rule, 
$\vs_{\nxt}= \vs_{\cur} + \alpha_{i,t} \, \vh_t^{\perp {\vs}_{\cur}}$ 
contrasted with the updated state resulting from the superposition recurrence update used in standard linear attention: $\hat{\vs}_{\nxt}= \vs_{\cur} + \alpha_{i,t} \, \vh_t$, (dashed arrow).
A unit writing intensity ($\alpha_{i,t}=1$) is assumed for simplicity in both recurrent update rules.
}
\end{figure}

\begin{highlightbox}
This proposition formalizes that by updating the memory slot with only the orthogonal component and then projecting back onto the unit sphere, we are effectively performing gradient descent on the Riemannian manifold of unit-norm state vectors.
\end{highlightbox}

\section{Experiment Details} \label{apdx:exp-details}

\textbf{Datasets:} 
We trained models on different datasets.  For language modeling and common-sense reasoning tasks, models are trained on FineWeb-Edu dataset~\citep{penedo2024fineweb} with context length of 4k. 
For sequence length scaling pattern models are trained on \method{The~Pile}  and \method{Books3} dataset.
\method{The~Pile} is a large-scale, diverse corpus widely used for training and evaluating language models~\citep{gao2020pile}. It consists of a mixture of high-quality text sources, including books, academic papers, web content, and technical documentation. 
While it contains relatively few sequences exceeding $8k$ tokens, in this study, we restrict The Pile to a short-context setting with sequence lengths of 2k or 8k tokens.
\method{Books3}, on the other hand,  is a subset of The Pile that consists of high-quality, full-length books, commonly used  for training language models for long-context evaluations. In the experiments we used this dataset  to test model performance on sequences ranging from 512 to 16k tokens (in increments of $ 2 \times $ per experiment). The same training setup as The Pile is applied to ensure consistency. Since Books3 contains structured narratives and long-form content, it provides a rigorous test of a model’s ability to track dependencies over extended contexts. 

For all experiments, the training batch size is fixed at 0.5 million tokens, irrespective of sequence length. This means that for a given context length $ T $, each batch contains $ 0.5M / T $ sequences.

\paragraph{Baseline Models and Model Architecture}  
We compare our method against \method{Transformer}++ model~\citep{touvron2023llama} as well as the following sub-quadratic sequence models: 
\method{Linear-Attention~(LA)}~\citep{katharopoulos2020transformers}, 
\method{TTT}~\citep{sun2024ttt},
\method{DeltaNet}~\citep{yang2024parallelizing},
\method{Gated~DeltaNet}~\citep{yang2024gatedDeltaNet},
\method{Mamba2}~\cite{mamba2}.
As discussed in the paper,  the \model layers incorporate $\ell_2$-normalization on the state, whereas \method{TTT} applies $\ell_2$-normalization on the output of the decoding layer.

\begin{figure}[t]
    \centering
    \includegraphics[trim= 25 50 140 20,clip,width=.7\linewidth]{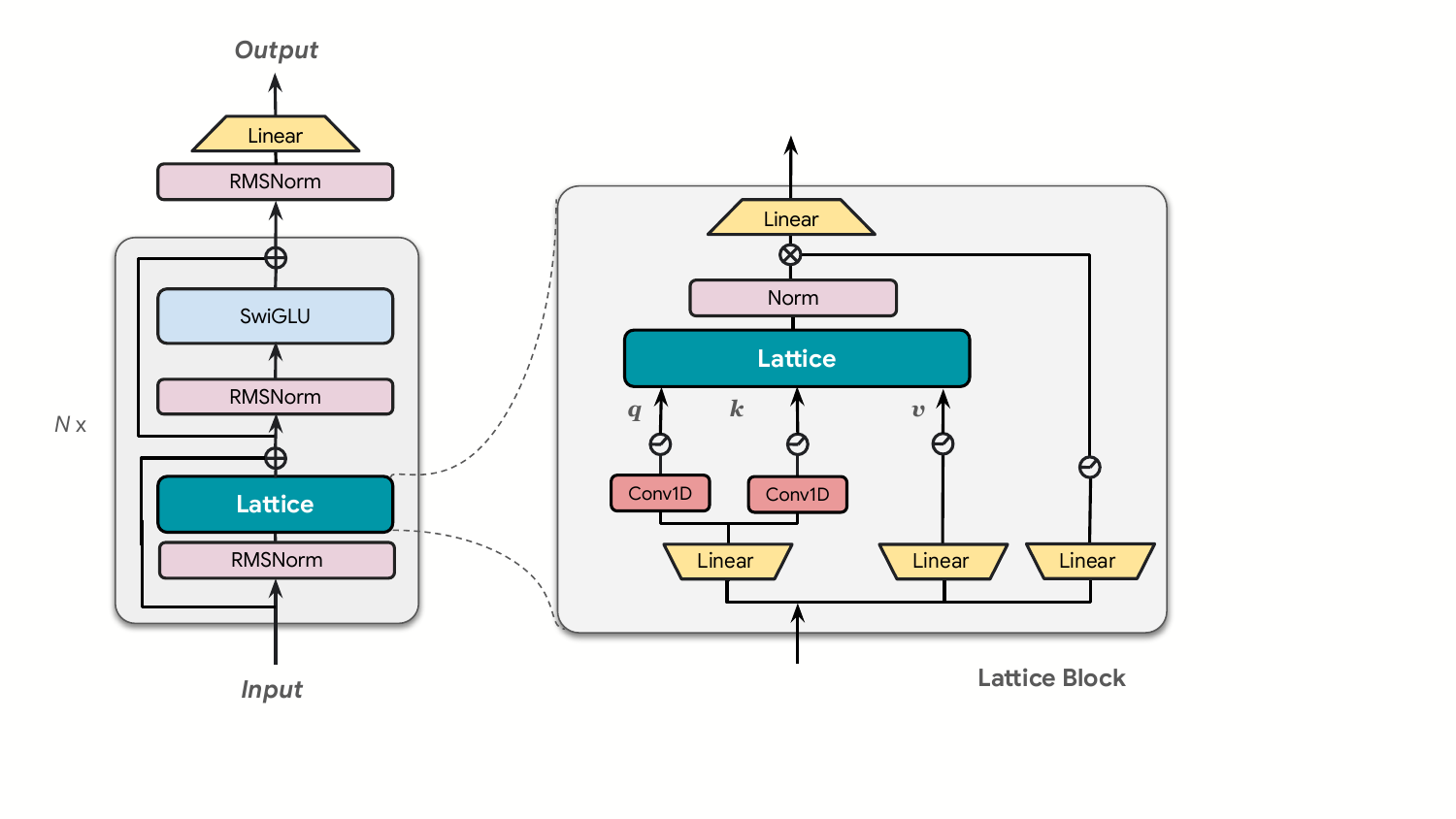}
    \vskip -5pt
  \caption{ \label{fig:model}
  \captionsize
      \figleft~
      Block diagram of the language model.
      \figright~
      The \model block. Following the architecture used in Mamba~\citep{gu2023mamba}, each sequence mixing block is composed of a pair of short $\texttt{Conv1D}$ for the pair $\{q, k \}$ and the \model is followed by a $\texttt{GeLU}$ post-gate.
  }
\end{figure}

\paragraph{Model Architecture} %
For sub-quadratic sequence models, we adopt the architectural setup used in Mamba~\citep{gu2023mamba}, where each sequence-mixing block consists of a pair of short \texttt{Conv1D} layers for the $\{\vq, \vk \}$ pair, which share a linear projection. A \texttt{GeLU} post-gate is applied to the output of the sequence model. The \method{Transformer}++ model, on the other hand, follows the architecture proposed in LLaMA~\citep{touvron2023llama}.
All the models follow the multi-head structure introduced in Transformers~\citep{vaswani2017attention}. The model architecture used for \model is illustrated in \autoref{fig:model}.

Instead of using a fixed learning rate as in standard gradient descent, we model the learning rate of the compression layer as an input-dependent neural network, $\gamma_t = \op{sigmoid}(\mathbf{W}_{\gamma} \vx_t)$. This is a key benefit of the bilevel optimization setup, as the weights of this network can be trained in the outer loop along with the rest of the model weights.

\paragraph{Hyperparameters.}
The block diagram of the language model is   in \autoref{fig:model}, and the model hyperparameters are listed in \autoref{tab:hyperparameters}.
We use the AdamW optimizer with a cosine learning rate schedule, which includes a
warm-up phase of 1k steps (0.5 billion tokens) and a final learning rate of $3\text{e-}5$. We apply a weight decay of 0.1 and gradient clipping at 1.0. The width of the short  convolution layer is set to 4.
For language model training on \method{FineWeb-Edu} and the reasoning tasks in \autoref{tab:FineWeb_340M_760_tasks}, a default chunk size of $C=4$ is used for \model unless otherwise specified.
For the ablation studies with small models in \autoref{tab:ablation} and the long-context scaling experiments (110M models, Books dataset) a chunk size of $C=1$ is used to evaluate the exact, non-approximated performance of the compression layer.

\begin{table}[htbp]
\centering

\begin{tabular}{l c c l l}
\toprule
\textbf{Configuration} & \textbf{$n_{\text{blocks}}$} & \textbf{$d_{\text{model}}$} & \textbf{$d_{\text{head}}$} & {Peak learning rate} \\
\midrule
110M params / 5B tokens & 12 & 768 & 64 & 1e-2 \\
340M params / 15B tokens & 24 & 1024 & 64 & 1.5e-3 \\
760M params / 30B tokens & 24 & 1536 & 64 & 1.25e-3\\
\bottomrule
\end{tabular}
\caption{Model and training hyperparameters.}
\label{tab:hyperparameters}
\end{table}

\begin{figure}
\centering
\begin{minipage}{.49\textwidth}
  \centering
  \begin{adjustbox}{width=.95\linewidth}
      \pgfplotsset{
    compat=1.16, %
    width=14cm, %
    height=9cm, %
    every axis plot/.append style={
        line width=1.5pt, %
        solid,            %
        mark size=2.5pt,  %
    }
}

\begin{tikzpicture}
    \begin{axis}[
        title={Books Dataset},
        xlabel={Context Length},
        ylabel={Perplexity  $\downarrow$},
        xmode=log, %
        log base x=2, %
        xtick={512, 1024, 2048, 4096, 8192, 16384},
        xticklabels={512, 1k, 2k, 4k, 8k, 16k},
        grid=major, %
        ymin=16.0, %
        ymax=22.8,
        legend columns=3,
        legend pos=north west, %
        legend cell align={left},
        grid=major, 
        ymode=linear,           %
    ]

    \pgfplotstableread{
    Context Transformer LinearAtt DeltaNet GatedDelta Mamba2 TTT ModelDEC ModelENC ModelSIM
    512     20.60     21.04     20.28    19.76   19.94   20.11    19.06   19.11   19.08
    1024    19.39     20.18     19.11    18.60   18.90   19.03    17.90   18.01   17.94
    2048    18.89     19.82     18.33    18.00   18.34   18.36    17.14   17.22   17.23
    4096    18.38     19.69     17.90    17.48   18.07   18.03    16.72   16.80   16.72
    8192    18.85     20.34     18.05    17.40   18.23   18.05    16.62   16.66   16.73
    16384   nan     21.86     18.12    17.49   18.48   18.46      16.97     nan     16.82
    }\datatablebooks

    \addplot+[color=cyan, mark=x] table [x=Context, y=Transformer] {\datatablebooks}; %
    \addlegendentry{\method{Transformer}++}

    \addplot+[color=blue, mark=o, dashed] table [x=Context, y=LinearAtt] {\datatablebooks}; %
    \addlegendentry{\method{Linear-Attention}}

    \addplot+[color=green, mark=triangle, solid] table [x=Context, y=DeltaNet] {\datatablebooks};
    \addlegendentry{\method{DeltaNet}}

    \addplot+[color=orange, mark=diamond] table [x=Context, y=Mamba2] {\datatablebooks};
    \addlegendentry{\method{Mamba2}}

    \addplot+[color=black, mark=x, solid] table [x=Context, y=GatedDelta] {\datatablebooks};
    \addlegendentry{\method{Gated-DeltaNet}}

    \addplot+[color=red, mark=square, solid] table [x=Context, y=TTT] {\datatablebooks};
    \addlegendentry{\method{TTT}}

    \addplot+[color=magenta, mark=star, solid] %
        table [x=Context, y=ModelDEC] {\datatablebooks};
    \addlegendentry{\textbf{\model-DEC}}
    
    \addplot+[color=brown, mark=star, solid] %
        table [x=Context, y=ModelENC] {\datatablebooks};
    \addlegendentry{\textbf{\model-ENC}}
    
    \addplot+[color=brown!50!black, mark=star, solid] %
        table [x=Context, y=ModelSIM] {\datatablebooks};
    \addlegendentry{\textbf{\model-SIM}}

    \end{axis}
\end{tikzpicture}
  \end{adjustbox}
\end{minipage}~\hfill{}~
\begin{minipage}{.49\textwidth}
  \centering
  \begin{adjustbox}{width=.95\linewidth}
    \pgfplotsset{
    compat=1.16, %
    width=14cm, %
    height=9cm, %
    every axis plot/.append style={
        line width=1.5pt, %
        solid,            %
        mark size=2.5pt,  %
    }
}

\pgfplotstableread{
Context Transformer LinearAtt DeltaNet GatedDelta Mamba2 TTT ModelDEC ModelENC ModelSIM
2048    11.58     12.56     11.62    11.31   11.51   11.59 10.88   10.90   10.89
8192    11.75     13.63     11.40    11.04   11.42   11.45 10.51   10.57   10.66
}\datatablepile

\begin{tikzpicture}
    \begin{axis}[
        title={The Pile Dataset},
        xlabel={Context Length},
        ylabel={Perplexity $\downarrow$},
        xmode=log, %
        log base x=2, %
        xtick={512, 1024, 2048, 4096, 8192, 16384},
        xticklabels={512, 1k, 2k, 4k, 8k, 16k},
        xmin=1500,
        xmax=12000,
        ymin=10.2, %
        ymax=14.5, %
        legend columns=3,
        legend pos=north west, %
        legend cell align={left},
        legend style={font=\small},
        grid=major,             %
        ymode=linear,           %
    ]

    \addplot+[color=cyan, mark=x] table [x=Context, y=Transformer] {\datatablepile}; %
    \addlegendentry{\method{Transformer}++}

    \addplot+[color=blue, mark=o, dashed] table [x=Context, y=LinearAtt] {\datatablepile}; %
    \addlegendentry{\method{Linear-Attention}}
    
    \addplot+[color=green, mark=triangle, solid] table [x=Context, y=DeltaNet] {\datatablepile};
    \addlegendentry{\method{DeltaNet}}

    \addplot+[color=orange, mark=diamond] table [x=Context, y=Mamba2] {\datatablepile};
    \addlegendentry{\method{Mamba2}}

    \addplot+[color=black, mark=x, solid] table [x=Context, y=GatedDelta] {\datatablepile};
    \addlegendentry{\method{Gated-DeltaNet}}

    \addplot+[color=red, mark=square, solid] table [x=Context, y=TTT] {\datatablepile};
    \addlegendentry{\method{TTT}}

    \addplot+[color=magenta, mark=star, solid] %
        table [x=Context, y=ModelDEC] {\datatablepile};
    \addlegendentry{\textbf{\model-DEC}}
    
    \addplot+[color=brown, mark=star, solid] %
        table [x=Context, y=ModelENC] {\datatablepile};
    \addlegendentry{\textbf{\model-ENC}}
    
    \addplot+[color=brown!50!black, mark=star, solid] %
        table [x=Context, y=ModelSIM] {\datatablepile};
    \addlegendentry{\textbf{\model-SIM}}

    \end{axis}
\end{tikzpicture}
  \end{adjustbox}
\end{minipage}
  \captionof{figure}{Model perplexity as a function of context length for models of size 110M parameters.
  \figleft~ displays results for the Books dataset vs context length  $\{512, 1024, 2k, 4k, 8k, 16k\}$  ; 
  \figright~ shows results for The Pile dataset vs context length  $\{2k, 8k\}$.
  Note that pre-training Transformers from scratch often performs poorly on very long contexts (e.g., 16k); the common approach is finetuning from shorter-context models~\citep{touvron2023llama}. 
  Therefore, the Transformer results shown here are limited to context lengths $T \le 8k$.
  }
    \label{fig:contextlength_Pile_Books}
\vskip -15pt
\end{figure}

\paragraph{Ablation: Scaling Memory Size.}
\autoref{fig:ablation_m} illustrates the scaling behavior of \model with respect to the number of memory slots $m$.
We observe that the performance of \method{\model} scales linearly with respect to $\log(m)$  (noting the logarithmic scale on the x-axis).
Notably, \model with only a quarter of the memory slots ($m=d/4=16$) outperforms \method{TTT} with full memory ($m=64$), which empirically validates that the proposed orthogonal update mechanism utilizes the fixed-state state capacity significantly more efficiently, enabling \model to achieve $4\times$ compression gain.
It is worth noting that while increasing $m$ yields consistent gains, models with $m > d$ come at the cost of increased parameter counts in the embedding projection matrices $\mathbf{W}_q$ and $\mathbf{W}_k$ (and vice versa for models with $m < d$).

\begin{figure}[htbp]
\centering

\definecolor{cLattice}{RGB}{23, 190, 207}   %
\definecolor{cTTT}{RGB}{214, 39, 40}        %
\definecolor{cDeltaNet}{RGB}{148, 103, 189} %

\begin{tikzpicture}
    \begin{axis}[
        width=8cm, height=5.5cm,
        xlabel={Number of Memory Slots $m$},
        ylabel={Perplexity $\downarrow$},
        xmode=log, %
        log base x=2, %
        grid=major,
        grid style={dashed, gray!30},
        legend pos=north east,
        xtick={32, 64, 96, 128, 160, 192},
        xticklabels={32, 64, 96, 128, 160, 192},
        xticklabel style={font=\scriptsize}, 
        ymin=14.5, ymax=16.5,
        yticklabel style={font=\scriptsize}, 
        every axis plot/.append style={thick, mark size=2.5pt}
    ]
    
    \addplot[
        color=cLattice,
        mark=pentagon*,
    ] coordinates {
        (16, 15.52)
        (32, 15.26)
        (64, 15.02)
        (96, 14.89)
        (128, 14.84)
        (160, 14.76)
        (192, 14.71)        
    };
    \addlegendentry{Lattice}

    \addplot[
        color=cDeltaNet,
        mark=otimes*,
        only marks,
        mark size=3pt
    ] coordinates {
        (64, 16.35)
    };
    \addlegendentry{DeltaNet}

    \addplot[
        color=cTTT,
        mark=star,
        only marks,
        mark size=3pt
    ] coordinates {
        (64, 15.60)
    };
    \addlegendentry{TTT}

    \end{axis}
\end{tikzpicture}
\caption{
\textbf{Ablation on Memory Size ($m$).}
Validation perplexity of \model{}'s (110M parameters), trained on \method{FineWeb-Edu} with varying number of memory slots $m$.
For comparison, \method{DeltaNet} and \method{TTT} are plotted at their standard configuration of $m=64$ (equivalent to head dimension $d=64$).
\method{\model} consistently outperforms baselines, achieving lower perplexity even with half the memory capacity.
}
\label{fig:ablation_m}
\end{figure}

\paragraph{Ablation: Per-token vs. Per-chunk Normalization.}
\autoref{tab:ablation_norm} compares the perplexity of the chunk-wise approximation \eq{eq:recurrence_linear1} using per-token versus per-chunk normalization strategies, where the latter applies normalization solely at chunk boundaries.
The results indicate that the performance divergence between the two strategies widens as the chunk size increases. Furthermore, for the simplified approximation \eq{eq:recurrence_linear2} with a chunk size of $C=16$, we observed that per-chunk normalization resulted in training instability. These results highlight the importance of per-token normalization for ensuring both performance and stability with larger chunk sizes. We attribute this to the fact that applying normalization exclusively at chunk boundaries creates a discrepancy between intra-chunk and inter-chunk update dynamics.

\begin{table*}
\centering
  \centering
    \footnotesize
\caption{ \label{tab:ablation_norm}
{
Ablation study of \model{}'s per-token vs. per-chunk normalization (110M parameters, trained on \method{FineWeb-Edu}). 
The perplexity of the chunk-wise approximation \eq{eq:recurrence_linear1} using per-token vs. per-chunk normalization strategies, where the normalization step is only applied after each chunk.
}
}
\setlength{\tabcolsep}{4pt}
\renewcommand{\arraystretch}{1.3}
\centering
\begin{tabular}{lc}
        \toprule
        \method{\textbf{Configuration}} & ppl $\downarrow$ \\
        \midrule
        \method{DeltaNet} & 16.35  \\    
        \method{TTT~C\!=\!1} & 15.60  \\        
        \arrayrulecolor{black!30}\midrule
        \method{\model~C\!=\!4} Per-token&  15.15 \\
        \method{\model~C\!=\!4} Per-chunk &  15.17 \\
        \method{\model~C\!=\!16} Per-token&     15.32 \\
        \method{\model~C\!=\!16} Per-chunk &    15.36 \\

    \arrayrulecolor{black}\bottomrule
\end{tabular}

    \vspace{-5pt}
\end{table*}

\subsection{{Recall-Intensive Task: Multi-Query Associative Recall}} \label{sec:exp_mqar}
{
To rigorously evaluate the memory capacity and  in-context learning capabilities of \model, we evaluate the model on Multi-Query Associative Recall (MQAR) task introduced by \citet{arora2023zoology}.

Prior research indicates that a model's capability in associative recall task is strongly predictive of its in-context learning quality and language modeling performance~\citep{olsson2022context}. 
While standard associative recall tests a model's ability to retrieve a value for a single query after a sequence of key-value pairs, MQAR introduced by \citet{arora2023zoology} as a more challenging and realistic scenario where the model must perform multiple recalls at varying positions within a single forward pass. 
In this task, a sequence consists of a series of key-value pairs followed by a series of queries (keys) mixed with other tokens. The model must correctly predict the value associated with a specific key. An example sequence is formatted as follows:
$$
\text{A 4 B 3 C 6 F 1 E 2} \dots \text{A ? C ? F ? E ? B ?} \rightarrow \text{4 6 1 2 3}
$$
In this setup, the model must compress the association of the key-value pairs in its memory state over potentially long distances and retrieve them correctly when queried multiple times.
Therefore, MQAR serves as a strong benchmark for the memory capacity of recurrent architectures while Transformer architecture can easily solve this task.

\paragraph{Experimental Setup.}
We follow~\citet{arora2023zoology} and adopt the experimental setup used in \citet{beck2024xlstm} for the more difficult setting. We generate 100,000 synthetic training samples and 3,000 validation samples from a vocabulary size of 8192. All models are trained for 64 epochs using the AdamW optimizer with a cosine annealing learning rate schedule after a 10$\%$ linear warm-up phase and with a weight decay of $0.1$. 

We train models with two blocks using a single head, while varying the Model Dimension ($d$) to observe scaling behaviors.
We set the number of slots $m=d$ for \method{\model}, \method{TTT}, \method{DeltaNet}, and \method{GLA}; consequently, each block maintains a state matrix of size $d^2$. In comparison, \method{xLSTM} uses an additional normalizer state vector, resulting in a larger state size of $d^2+d$ per block.

We evaluate \model on two challenging settings across different context lengths ($L$)  and high numbers of key-value pairs. For each setting, we sweep over learning rate grids to ensure optimal convergence:
\begin{itemize}
\item {Context Length $L$=1024:} We sweep learning rates over $\{1\text{e-}2, 3.16\text{e-}3, 1\text{e-}3,  3.16\text{e-}4, 1\text{e-}4\}$ with a batch size=96.
\item {Context Length $L$=2048:} We sweep learning rates over $\{1\text{e-}2, 1\text{e-}3, 2.2\text{e-}4, 5\text{e-}5, 1\text{e-}5\}$ with a batch size=24.
\end{itemize}

For each context length, we scale the difficulty by varying the number of Key-Value pairs that must be memorized, testing $N_{KV} \in \{48, 96, 256\}$. 
This progression specifically evaluates  memory capacity of the model, determining whether the architecture can effectively store and recall a high density of information within the fixed-size state, thereby testing the efficiency of the compression mechanism proposed by \model.

As the results in \autoref{fig:MQAR_hard} show, \model offers the best MQAR accuracy over the majority of the settings and it exhibits a significant improvement over its main counterparts: \method{TTT} and \method{DeltaNet}. Furthermore, while baselines exhibit performance degradation as context length increases, \model maintains its accuracy across both context lengths. 
\textit{This overall shows the efficient compression mechanism of \model, demonstrating its superior capability to store and retrieve dense information within its fixed-size recurrent memory.}
}

\input{FigureTable/MQAR_large}

\subsection{Throughput Comparison.} 
\begin{figure}[htbp]
\centering
\definecolor{cTransformer}{RGB}{214, 39, 40}   %
\definecolor{cTTT}{RGB}{31, 119, 180}          %
\definecolor{cLattice}{RGB}{44, 160, 44}       %

\pgfplotsset{
    throughputstyle/.style={
        width=7cm, height=5.5cm,
        xlabel={Sequence Length},
        xmode=log,
        log basis x=2,
        xtick={2048, 4096, 8192, 16384},
        xticklabels={2k, 4k, 8k, 16k},
        ymin=0, ymax=350, %
        grid=major,
        grid style={dashed, gray!30},
        every axis plot/.append style={thick, mark size=2.5pt},
        legend style={at={(0.5, -0.25)}, anchor=north, legend columns=-1},
    }
}

\begin{tikzpicture}
    \begin{groupplot}[
        group style={
            group size=2 by 1,
            horizontal sep=1.5cm,
            vertical sep=1.5cm
        },
        throughputstyle
    ]

    \nextgroupplot[
        title={\textbf{Chunk Size = 64}},
        ylabel={Throughput (k tokens/sec)},
        legend to name=SharedLegend
    ]

    \addplot[color=cTransformer, mark=square*] coordinates {
        (2048, 341.462)
        (4096, 211.193)
        (8192, 119.754)
        (16384, 65.136)
    };
    \addlegendentry{Transformer++}

    \addplot[color=cTTT, mark=*] coordinates {
        (2048, 260.307)
        (4096, 243.906)
        (8192, 220.304)
        (16384, 172.597)
    };
    \addlegendentry{TTT}

    \addplot[color=cLattice, mark=triangle*] coordinates {
        (2048, 233.099)
        (4096, 211.111)
        (8192, 192.537)
        (16384, 158.582)
    };
    \addlegendentry{\model}

    \nextgroupplot[
        title={\textbf{Chunk Size = 16}},
        ylabel={} %
    ]

    \addplot[color=cTransformer, mark=square*] coordinates {
        (2048, 341.462)
        (4096, 211.193)
        (8192, 119.754)
        (16384, 65.136)
    };

    \addplot[color=cTTT, mark=*] coordinates {
        (2048, 157.113)
        (4096, 162.459)
        (8192, 120.084)
        (16384, 74.394)
    };

    \addplot[color=cLattice, mark=triangle*] coordinates {
        (2048, 123.583)
        (4096, 109.989)
        (8192, 91.262)
        (16384, 69.707)
    };

    \end{groupplot}

    \node at ($(group c1r1.south)!0.5!(group c2r1.south) + (0,-1.3cm)$) {\ref{SharedLegend}};
\end{tikzpicture}
\caption{
{\textbf{Training Throughput vs. Sequence Length.}
Throughput (in k tokens/sec) for 110M parameter models trained on a 2$\times$2 TPU Trillium setup.
\textbf{Left:} Throughput using a Chunk Size of $64$.
\textbf{Right:} Throughput using a Chunk Size of $16$.
Due to its recurrent formulation, \method{\model} maintains higher throughput than Transformer++ as sequence length increases, remaining competitive with 
\method{TTT}.
Lower chunk sizes generally reduce throughput due to lower FLOP utilization and reduced parallelism.
}
}

\label{fig:throughput}
\end{figure}
The training throughput are compared in \autoref{fig:throughput}. For 
\method{Transformer}++, we use Flash-Attention-2~\citep{dao2023flashattention}  with a block size of 256, while the rest of the models use a chunk size of 64.  
The results demonstrate that \method{\model} exhibits superior throughput scalability compared to Transformers as sequence length increases, while remaining competitive with \method{TTT}. Our results confirm that \method{\model} maintains the favorable sub-quadratic scaling characteristic of RNNs while offering improved expressivity. Both \method{\model} and \method{TTT} are implemented in pure \method{Jax}, without relying on specialized hardware-optimized kernels.

\clearpage
\section{Implementation}
\label{apdx:implementation}
\begin{lstlisting}[language=Python,style=mystyle,caption={{A simple implementation of the \model update rule.}
}]

from functools import partial
import jax
import jax.numpy as jnp
from flax import linen as nn
from jax import lax, vmap
from einops import rearrange
from jax.numpy.linalg import norm 

@partial(jax.jit, static_argnames=['mode', 'chunk_size'])
def lattice_compress(S0, K, V, Q, eta, mu, mode='Dec', chunk_size=16):
    """
    A simple Lattice operation with recursive update.
    S0: (BxHxMxD)
    K: (BxLxHxM), V: (BxLxHxD), Q: (BxLxHxM), eta/mu: (BxLxHx1)
    """
    # split chunks and transpose to B H NC C d
    K = rearrange(K, 'b (nc c) h m -> b h nc c m', c=chunk_size)
    V = rearrange(V, 'b (nc c) h d -> b h nc c d', c=chunk_size)
    Q = rearrange(Q, 'b (nc c) h m -> b h nc c m', c=chunk_size)
    eta = rearrange(eta, 'b (nc c) h 1 -> b h nc c 1', c=chunk_size)
    mu = rearrange(mu, 'b (nc c) h 1 -> b h nc c 1', c=chunk_size)

    if chunk_size == 1:  # scan over sequence
      scan_fn = partial(compress_step, mode=mode)
    else:  # scan over chunks
      scan_fn = partial(compress_chunk, mode=mode)
    
    def compress_scan(s0, k,v,q,e,m):
        return lax.scan(scan_fn, s0, (k,v,q,e,m))

    # vmap over batch and heads
    batch_head_scan_fn = vmap(
        vmap(compress_scan, axis_name="heads"),
        axis_name="batch")
    S, Y = batch_head_scan_fn(S0, K, V, Q, eta, mu)
    Y = rearrange(Y, 'b h nc c d -> b (nc c) h d')
    return S, Y
  
def compress_step(S, inputs_t, mode='Dec'):
  # S: MxD
  # inputs_t: k:1xM, v:1xD, q:1xM, eta:1x1, mu:1x1
  k, v, q, eta, mu = inputs_t

  s_norm = norm(S, axis=-1, keepdims=True) + 1e-6
  Phi = S / s_norm 
  
  if mode == 'Dec':
    h = k @ Phi - v
  elif mode == 'Sim':
    h = - v

  # 1) orthogonal projection (* $\vh_t^{\perp \vs_i}$ \eq{eq:similarity_recurrence}, \eq{eq:OSR} *): compute Delta  (* \Eqref{eq:recurrence_detailed} *) 
  h_hat= (h @ Phi.T) / s_norm.T
  k_hat = k / s_norm.T
  Delta1 = k_hat.T @ h
  Delta2 = S * (h_hat * k_hat).T
  Delta = Delta1 - Delta2
  
  # #VERIFY ORTHOGONALITY of Delta and S
  # jax.debug.print("<Delta,S>: {}",jnp.einsum('md, md->m',Delta,S))
    
  # 2) update S and normalize it (* $\vbeta_t$ *) 
  S_t = mu * S - eta * Delta
  beta = s_norm * jax.lax.rsqrt(
      ((mu*S)**2 + (eta * Delta)**2).sum(axis=-1, keepdims=True))
  # beta = s_norm / (norm(S_t, axis=-1, keepdims=True)+ 1e-6)
  S_t *= beta       # (* (\Eqref{eqn:normalized_recurrence}) *) 
  # readout
  y = q @ S_t
  return S_t, y

def compress_chunk(S, inputs_t, mode='Dec'):
  # S: MxD
  # inputs_t: K:CxM, V:CxD, Q:CxM, eta:Cx1, mu:Cx1
  K, V, Q, eta, mu = inputs_t
  
  s_norm = norm(S, axis=-1, keepdims=True) + 1e-6 # Mx1
  Phi = S / s_norm # MxD
  
  if mode == 'Dec':
    H = K @ Phi - V  # CxD
  elif mode == 'Sim':
    H = - V  # CxD
  H_hat= (H @ Phi.T) / s_norm.T  # CxM : (* \Eqref{eq:recurrence_detailed} *)
  K_hat = K / s_norm.T  # CxM
  
  # compute (* $\| \Delta \mS \|^2$ and $\vbeta$ \Eqref{eq:beta_chunk} *)
  Delta1_sq = K_hat**2 * (H*H).sum(axis=-1)[:, None] # CxM
  Delta2_sq = (s_norm.T**2) * (H_hat * K_hat)**2     # CxM
  Delta_sq = jnp.clip(Delta1_sq - Delta2_sq, min=0.) # CxM
  Delta_sq = eta**2 * Delta_sq
  beta = s_norm.T * jax.lax.rsqrt(eta**2 * s_norm.T**2 + Delta_sq)   # (CxM)
  
  G_gate = beta * (mu+ eta*(H_hat*K_hat)) # CxM : (* \Eqref{eq:recurrence_linear2} *)
  
  # Parallel intra-chunk based on linear attention 
  # with vectorized forgetting gate (GLA): (* \Eqref{eq:recurrence_linear2} *)
  K_tilde = -beta * eta * K_hat  # CxM
  S_t, Y = GLA_intra_chunk(Q=Q, K=K_tilde, V=H, G=G_gate, S0=S)
  return S_t, Y


\end{lstlisting}

\paragraph{Limitation and Future Work.}
The high expressiveness of the proposed orthogonal projection comes at the cost of an extra   $\mathcal{O}(d\,m)$ computation compared to simpler linear RNNs, as detailed in section \ ref{sec:complexity}. 
However, the final update rule, presented in its chunk-wise form in sections \ref{sec:complexity} and \ref{app:parallel_details}, relies entirely on standard, hardware-efficient operators such as matrix and element-wise multiplications.
While the proposed chunk-wise approximations in \autoref{eq:recurrence_linear1} and \autoref{eq:recurrence_linear2} enable parallelizable matrix multiplications within chunks, the non-linear nature of this test-time training prevents the use of inter-chunk parallelization via associative scans, which limits hardware utilization compared to fully linear recurrences.
We believe that exploring hardware-efficient strategies remains a key area for future work. Specifically, tailoring techniques such as the large-chunk training strategies proposed in \method{LaCT}~\citep{zhang2025TTTDR}, or the multi-stage and hierarchical process proposed in~\citet{li2025tnt} for the non-linear recurrence of \model represents a promising direction to maximize hardware utilization while retaining the benefits of the principled orthogonal update.

\clearpage

\end{document}